\documentclass[letterpaper]{article}
\usepackage{aaai2026} 
\usepackage{times} 
\usepackage{helvet}  
\usepackage{courier}  
\usepackage[hyphens]{url}  
\usepackage{graphicx}
\usepackage{natbib}
\usepackage{caption} 
\usepackage{algorithm}
\usepackage{algorithmic}
\usepackage{newfloat}
\usepackage{listings}
\DeclareCaptionStyle{ruled}{labelfont=normalfont,labelsep=colon,strut=off}
\floatstyle{ruled}
\newfloat{listing}{tb}{lst}{}
\floatname{listing}{Listing}

\usepackage{amsmath} 
\usepackage{amssymb, amsthm}
\usepackage{multirow}
\newtheorem{theorem}{Theorem}
\newtheorem{proposition}[theorem]{Proposition}
\newtheorem{lemma}[theorem]{Lemma}

\newtheorem{example}[theorem]{Example}
\theoremstyle{definition}
\newtheorem{definition}[theorem]{Definition}

\theoremstyle{remark}

\title{Test-time Diverse Reasoning by Riemannian Activation Steering}

\author {
    Ly Tran Ho Khanh\equalcontrib\textsuperscript{\rm 1},
    Dongxuan Zhu\equalcontrib\textsuperscript{\rm 1},
    Man-Chung Yue\textsuperscript{\rm 2},
    Viet Anh Nguyen\textsuperscript{\rm 1},
}
\affiliations {
    \textsuperscript{\rm 1}The Chinese University of Hong Kong \quad
    \textsuperscript{\rm 2}The University of Hong Kong\\
    hokhanhlytran@cuhk.edu.hk, dxzhu@se.cuhk.edu.hk, mcyue@hku.hk, nguyen@se.cuhk.edu.hk
}

\usepackage{bibentry}
\newcommand{\dd}{\mathrm{d}}
\newcommand{\R}{\mathbb{R}}
\newcommand{\mc}{\mathcal}

\newcommand{\opt}{^\star}

\newcommand{\grad}{\mathrm{grad}}
\newcommand{\Proj}{\mathrm{Proj}}
\newcommand{\Exp}{\mathrm{Exp}}
\newcommand{\Hess}{\mathrm{Hess}}
\newcommand{\Let}{\triangleq}

\usepackage{xspace}

\begin{document}

\maketitle

\begin{abstract}
Best-of-$N$ reasoning improves the accuracy of language models in solving complex tasks by sampling multiple candidate solutions and then selecting the best one based on some criteria. A critical bottleneck for this strategy is the output diversity limit, which occurs when the model generates similar outputs despite stochastic sampling, and hence recites the same error. To address this lack of variance in reasoning paths, we propose a novel unsupervised activation steering strategy that simultaneously optimizes the steering vectors for multiple reasoning trajectories at test time. At any synchronization anchor along the batch generation process, we find the steering vectors that maximize the total volume spanned by all possible intervened activation subsets. We demonstrate that these steering vectors can be determined by solving a Riemannian optimization problem over the product of spheres with a log-determinant objective function. We then use a Riemannian block-coordinate descent algorithm with a well-tuned learning rate to obtain a stationary point of the problem, and we apply these steering vectors until the generation process reaches the subsequent synchronization anchor. Empirical evaluations on popular mathematical benchmarks demonstrate that our test-time Riemannian activation steering strategy outperforms vanilla sampling techniques in terms of generative diversity and solution accuracy.
\end{abstract}

\begin{links}
    \link{Code}{https://github.com/lythk88/SPREAD}
\end{links}

\section{Introduction}

Language models (LMs) have revolutionized tasks from code generation~\cite{ref:chen2021evaluating}, symbolic reasoning~\cite{ref:wei2023chain}, to mathematical problem solving~\cite{ref:lewkowycz2022solving}. In these tasks, the quality of the final solution can improve significantly by exploring multiple plausible reasoning paths and then presenting the best solution aggregated from the information from these paths. This strategy aligns with our problem-solving intuitions, where each path of reasoning explores different techniques, generalizability, and scientific values. A simple and effective method to exploit this explore-then-aggregate idea is Best-of-$N$ sampling~\cite{ref:lightman2023let, ref:ni2023lever}, where the model generates $N$ candidates and uses a reward model to select the most plausible answer. The final accuracy of this Best-of-$N$ strategy is constrained by the diversity of the generated candidates. 
A straightforward method to encourage exploration is to use a stochastic sampling decoder when generating the next token, leading to a zoo of emerging methods~\cite{ref:fan2018hierarchical,ref:holtzman2020curious,ref:meister2025locally}. These methods construct a token-level search space with varying access to the internals of an LM, such as logits, next-token distributions, or probability scores.

Yet, stochastic decoding methods frequently suffer from \textit{diversity collapse}, where the outputs of LMs may converge to nearly identical reasoning paths~\cite{ref:yun2025the, ref:dang2025weight}. This phenomenon has sparked more aggressive search strategies, such as contrastive search \cite{ref:su2023contrastive}, balancing the model confidence with the degeneration penalty to avoid repetitiveness. Alternatively, \citet{ref:vijayakumar2018diverse} maintains multiple diverse hypotheses during beam search by diversity-promoting objectives. Additionally, self-speculative decoding \cite{ref:Zhang2024Draft} and speculative sampling \cite{ref:leviathan2023fast} inherently promote diversity through their multi-step prediction mechanisms. Most of these search strategies are computationally intensive, as they require joint consideration of reasoning trajectories distribution~\cite{ref:nagarajan2025roll}.

Another challenge to promoting reasoning diversity is measuring it. Popular metrics like lexical or semantic diversity may not capture the reasoning diversity well. Lexical diversity measures the number of different meanings or perspectives conveyed among sequences, but is sensitive to text length, rephrasing, or synonyms~\cite{ref:bestgen2024measuring}. Similarly, semantic diversity quantifies the different meanings or perspectives, but it is sensitive to adding or removing details from the text~\cite{ref:han2022measuring}. Computationally, both often invoke extra neural architecture for evaluation, thus adding load to the inference process.

In this paper, we look at the reasoning diversity through another proxy: the diversity of the hidden activations of the generated sequences. Hidden activations are the internal representations computed by a language model at each layer and token position as it processes input. They encode intermediate computations and abstract concepts, acting as the latent ``thinking space" where reasoning, memory, and structure are implicitly formed. While there may be strong correlations between the activations and the reasoning paths, there is unfortunately no one-to-one equivalence between them. Thus, admittedly, promoting diversity of the hidden activations does not necessarily lead to diversity in the reasoning paths. However, recent progress suggests that encouraging diversity among neuron activations within the same layer increases the capacity of the model to learn a broader range of features~\cite{ref:laakom2023wldreg}. In general, increasing the diversity of hidden activations can reduce estimation error and improve generalization.
Moreover, models that facilitate diverse internal activations may be better equipped to represent and synthesize multiple reasoning strategies, as the richer internal space allows for more varied ``thought paths"~\cite{ref:naik2024diversity}. Recent results in interpretability indicate that different features or activation clusters can sometimes correspond to different ``reasoning circuits" or strategies, especially in LMs~\cite{ref:jack2025on}. These observations suggest that we should design fast and parameter-efficient mechanisms to promote the diversity of hidden activations during generation, hoping to induce reasoning diversity and improve the accuracy for Best-of-$N$ sampling.

\noindent\textbf{Contributions.} We summarize our contributions as follows:
\begin{itemize}
    \item We propose the \textbf{SP}herical intervention for \textbf{REA}soning \textbf{D}iversity (SPREAD), an unsupervised activation steering method that improves the diversity among reasoning trajectories. At a synchronization anchor, SPREAD extracts the hidden activations from all sequences, then computes the steering vectors that maximize the total volume spanned by all possible subsets of the intervened activations. SPREAD then adds these steering vectors to the respective activations of all subsequent tokens until the next synchronization anchor. 
    \item We show that determining the optimal steering vectors can be reformulated as a manifold optimization problem defined over the product of spheres, where the log-determinant objective function captures the geometric diversity of the intervened activations. We propose using a Riemannian block coordinate descent algorithm, which exploits the product structure of the manifold constraints. We also study the theoretical properties of the optimization problem and prove the convergence guarantee of the algorithm for appropriate step sizes.
\end{itemize}

Using the steering methods, SPREAD uses readily available hidden activations from the generation process and does not require any additional neural architectures to measure quality or reasoning diversity. Moreover, SPREAD could rely on only one hyperparameter that prescribes the relative radii of the intervention vectors, and it relieves the burden of parameter tuning at inference time.

Our paper unfolds as follows: Section~\ref{sec:review} reviews the related works on generative diversity and activation steering, Section~\ref{sec:SPREAD} presents the mathematical formulation of the SPREAD framework, Section~\ref{sec:manifold opt} develops the manifold optimization algorithm for computing the optimal steering vectors, and Section~\ref{sec:experiment} empirically illustrates the performance of SPREAD on mathematical reasoning tasks. 

\textbf{Notations.} The space of $p$-dimensional vectors is denoted $\R^{p}$. For any $x\in \R^{p}$, $\| x \|_2$ is its Euclidean norm. For a matrix $A \in \R^{p \times N}$, we use $\|A\|_{F}$ for the Frobenius norm. We use $\nabla\ell(V)$ and $\nabla^2 \ell(V)$ for the gradient of function $\ell$ and Hessian matrix of function $\ell$ with respect to $V$ in the Euclidean sense; while $\grad~\ell$ and $\Hess~\ell$ are the Riemannian counterparts. We use $\nabla_{i}\ell(V)$, $\nabla^{2}_{i}\ell(V)$, $\grad_{i}~\ell(V)$ and $\Hess_{i}~\ell(V)$ for the corresponding operator with respect to the $i$-th block $v_{i}$ while fixing all other blocks. All proofs are relegated to the appendix.

\section{Literature Review} \label{sec:review}
\textbf{Diversity in generation.} 
Classical approaches, including temperature sampling \cite{ref:ackley1985learning}, top-$k$ sampling \cite{ref:fan2018hierarchical}, nucleus sampling \cite{ref:holtzman2020curious}, and typical decoding \cite{ref:meister2025locally}, promote output diversity by introducing stochasticity into the generation process. 
Prompt‐centric techniques have also been shown to enrich the diversity of reasoning. \citet{ref:li2023making} focuses on prompt diversity by generating diverse prompts to explore different reasoning paths, while filtering out incorrect answers by a weighted voting scheme. \citet{ref:naik2024diversity} proposes a self‐reflective prompting that leverages the LLM as a guide to design a diverse set of approaches for complex reasoning tasks. 
\citet{ref:Wang2025adaptive} uses multiple adaptive steering vectors for different hallucination types, though maintaining multiple vectors is computationally expensive. \citet{ref:chung2025revisiting} achieves diversity through parameter-efficient prefix tuning, but effectiveness is sensitive to training data quality.

\noindent\textbf{Activation Steering} is a lightweight and interpretable method for controlling LMs. This approach injects direction vectors into the residual stream of transformer layers to steer generation toward desired attributes (e.g., truthfulness, sentiment, toxicity) without modifying model weights.  \textit{Contrastive steering methods} derive directions by comparing activations between positive and negative examples of desired behaviors. \citet{ref:turner2023steering} computes steering vectors by averaging residual stream differences between factual and hallucinatory responses, while \citet{ref:stolfo2025improving} contrasts activations with and without specific instructions. \textit{Probe-based steering methods} instead learn to identify relevant concepts through trained classifiers, then extract steering directions from the learned representations~\cite{ref:Li2023advances, ref:zhang2025controlling}.
Despite demonstrating effectiveness across domains like toxicity reduction~\cite{ref:zhang2025controlling}, and truthfulness enhancement~\cite{ref:turner2023steering}, activation steering restricts its broader applicability. Most existing approaches rely on single, fixed direction vectors that constrain model outputs to narrow behavioral modes. 
The challenge becomes even more pronounced in mathematical reasoning, where defining clear positive and negative exemplars for contrastive learning is inherently difficult because mathematical correctness involves complex, context-dependent logical structures that resist simple binary classification. These domain-specific challenges explain why prior activation steering research has largely avoided mathematical reasoning applications.

\begin{figure*}[!ht]
    \centering
    \includegraphics[width=1.05\textwidth]{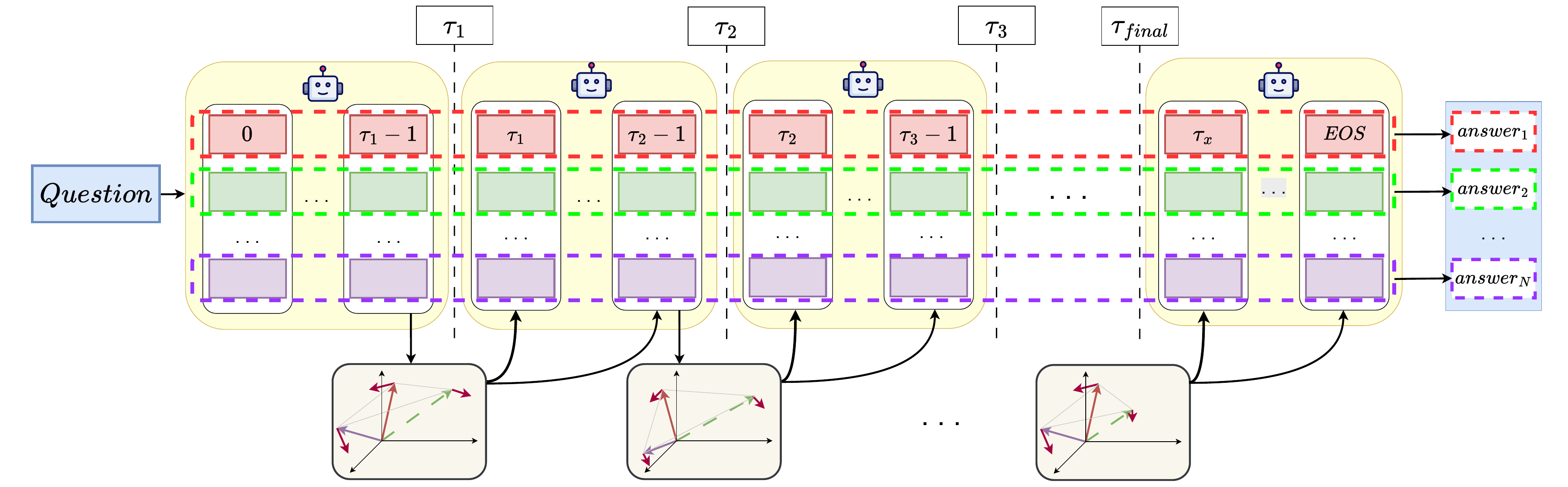}
    \caption{Overview of SPREAD for generating $N$ diverse reasoning answers simultaneously. At each decoding step $\tau_t$, we extract the hidden vectors corresponding to the last token in each path.  These hidden vectors serve as inputs to Algorithm~\ref{alg:algorithm1}, where they are projected into a shared activation space to compute $N$ steering vectors. This process is repeated until an end-of-sequence $EOS$ token is generated.}
    \label{fig:framework}
    \vspace{-3mm}
\end{figure*}

\section{Activation Steering for Diverse Generation under the Optimization Lens} \label{sec:SPREAD}

We aim to elicit diverse reasoning paths from an LM at test time without any fine-tuning by sampling multiple sequences per input prompt. Figure \ref{fig:framework} shows the core idea of the SPREAD framework, which intervenes directly in the model's activation space during the autoregressive generation process.
Specifically, we simultaneously generate $N$ output sequences and extract the final hidden state vectors after $\tau$ tokens, $H = [h_1, \ldots, h_N] \in \R^{p \times N}$, for all sequences. 
We then compute an additive steering vector $v_i$ for each hidden state $h_i$, yielding a new set of hidden states $H_{new} = H + V$, where $V=[v_1, \ldots, v_N]$. The core of our approach is to maximize a geometric measure of diversity for $H_{new}$ (\ref{fig:placeholder}), or equivalently with respect to the steering vectors $V$. A natural measure for the dispersion of a set of vectors is the volume of the parallelepiped they span.
\begin{definition}[Parallelepiped] Given $n$ vectors $h_1, \ldots, h_n \in \R^p$, the parallelepiped is their convex hull:
    \[
        \mathcal P( \{ h_1, \ldots, h_n\} ) \Let \left\{ h \in \R^p : h = \sum_{i = 1}^n \lambda_i h_i,~\lambda \in [0, 1]^n \right\}.
    \]    
\end{definition}
To encourage robust diversity, we require that not only the entire set of $N$ steered vectors be diverse, but that \textit{any} subset of these vectors also be diverse. This prevents degenerate solutions where, for instance, $N-1$ vectors are clustered together and only one is pushed far away. Let $\mathbb I \in 2^{[N]}$ be any index subset of the $N$ sequences. We aim to maximize the sum of the squared volumes of the parallelepiped associated with all possible subsets:
    \begin{equation} \label{eq:volume}
        \max_{\forall i: \| v_i \|_2^2 \le \alpha_i}~\sum_{\mathbb I \in 2^{[N]}} \mathrm{Volume}( \mathcal P(\{ h_i + v_i\}_{i \in \mathbb I}) )^2,
\end{equation}
where $\alpha_i$ is a magnitude for the intervention vector $v_i$ on the $i$-th path. The squared norm constraints effectively keep the modified hidden state $h_{i}+v_{i}$ within a ``trust region" of moderate radius $\sqrt{\alpha_i}$ around the original state $h_{i}$. If $\alpha_i$ is too big, the vector $v_i$ could erase meaningful information stored in $h_i$, leading to generation collapse. The next proposition gives an explicit form of the objective function in Problem~\eqref {eq:volume} as a log-determinant function.

\begin{proposition}[Objective function equivalence] \label{prop:obj}
    Problem~\eqref{eq:volume} is equivalent to the following log-determinant optimization problem
    \begin{equation} \label{eq:minlogdet}
        \min_{\substack{\forall i: \| v_i \|_2^2 \le \alpha_i \\ V = [v_1, \ldots, v_N]}}~\ell(V) \Let -\log\det [ I + (H+V)^\top (H + V) ].
    \end{equation}
\end{proposition}

Problem~\eqref{eq:minlogdet} has a convex feasible set, but its objective function $\ell$ is non-convex, as illustrated in the next example.
\begin{example}[Non-convexity of $\ell$]
        Take \[
        H=\left[\begin{matrix}
            1&0\\0&1
        \end{matrix}\right], V_1 = \left[\begin{matrix}
            0&0\\0&0
        \end{matrix}\right], V_2 = \left[\begin{matrix}
            -2&0\\0&-2
        \end{matrix}\right],
        \] and $V_3 = \frac{1}{2}(V_{1} + V_{2})$. Then  $\ell(V_{1})=\ell(V_{2})=-\log 4$ and $\ell(V_{3})=0$. We find $2\ell(V_{3}) - ( \ell(V_{1})+\ell(V_{2})) =2\log4>0$, which implies that $\ell$ is not convex.
    \end{example}

Problem~\eqref{eq:minlogdet} turns out to be a non-convex problem, and it is, in general, NP-hard to find its global optimum~\cite{ref:jin2021nonconvex}. However, we can show a qualitative result asserting that the optimal steering vectors of Problem~\eqref{eq:volume} will make the norm constraint $\|v_{i} \|^{2}_{2} \leq \alpha_{i}$ binding, and we obtain another equivalent problem with equality constraints.

\begin{proposition}[Constraint equivalence]\label{prop:equality_refor}
    Problem~\eqref{eq:volume} is further equivalent to the following log-determinant optimization problem with equality constraints 
    \begin{equation}\label{eq:minlogdet_eq}
     \min~\left\{ \ell(V)~:~\| v_i \|_2^2 = \alpha_i, V = [v_1, \ldots, v_N] \right\}.
    \end{equation}
\end{proposition}

The equality constraints $ \| v_i \|_2^2 = \alpha_i$ are no longer convex. However, the advantage of the reformulation~\eqref{eq:minlogdet_eq} is that we can cast Problem~\eqref{eq:minlogdet_eq} as an optimization problem over a Riemannian manifold. In the next section, we will describe the procedure for solving Problem~\eqref{eq:minlogdet_eq} with manifold optimization methods.

    \noindent\textbf{Hyperparameters.} Problem~\eqref{eq:volume} and its equivalent form~\eqref{eq:minlogdet_eq} requires $N$ radii values $\{ \alpha_i\}_{i=1}^N$ as input hyperparameters. We propose to set $\alpha_i = C \| h_i \|_2 / p$ for all $i$, where $p$ is the dimension of $h_i$, and thus the number of hyperparameters is reduced to only one relative parameter $C > 0$.

\section{Manifold Optimization for Steering}\label{sec:manifold opt}
This section is to devise an efficient algorithm for solving Problem~\eqref{eq:minlogdet_eq} based on Riemannian optimization. Formally, we let $\mc M_i$ be the sphere in $\R^{p}$ of radius $\sqrt{\alpha_i}$:
    \[
        \mc M_i = \{ v_i \in \R^{p} ~:~ \| v_i \|_2^2 = \alpha_i \},
    \]
and define the product manifold $\mc M= \mc M_{1} \times \cdots \times \mc M_{N}$. Problem~\eqref{eq:minlogdet_eq} can be cast as a Riemannian optimization problem over the product manifold $\mc M$, which is naturally solved using Riemannian optimization algorithms. 
\begin{figure}
    \centering
    \includegraphics[width=0.7\linewidth]{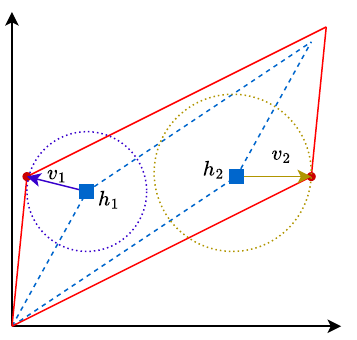}
    \caption{An illustration of the volume maximization intuition behind SPREAD. The hidden vectors $h_1$ and $h_2$ (originally blue squares) are pushed toward target positions using the corresponding steering vectors $v_1$, $v_2$, found via Riemannian Block Coordinate Descent (Algorithm~\ref{alg:algorithm1}). After intervention, the new parallelepiped (red) has a larger volume than the original parallelepiped (dashed blue). }
    \label{fig:placeholder}
\end{figure}
\subsection{Preliminaries about the Manifold $\mc M$}

Riemannian optimization exploits the geometric structure of the constraint set to perform updates along curved spaces (manifolds) rather than flat Euclidean space. Key tools include (i) the Riemannian gradient, which generalizes classical gradient methods to manifold settings, and (ii) the exponential map, which generalizes the concept of moving in a straight line along a gradient in Euclidean space to moving along a shortest path on a manifold.

At a point on a manifold, a tangent space is a vector space that ``touches" the manifold at that point and contains all possible directions in which one can move along the manifold from that point. For an individual sphere, the tangent space at $v_i \in \mathcal M_i$ is
    \[
        T_{v_i} \mc M_i = \{ z \in \R^{p} : z^\top v_i = 0 \}.
    \]
For the product manifold, the tangent space at $V = [v_1, \ldots, v_N]$ is  \[
T_V \mc M = \{Z=(z_1, \dots, z_N)~:~z_i \in T_{v_i} \mc M_i~ \forall i\}.
\]
The Riemannian gradient of $\ell$ defined on the manifold $\mathcal M$ is the steepest ascent direction that lies within the tangent space of the manifold at a given point. It is computed by first taking the Euclidean gradient of $\ell$ in the ambient space $\R^{p \times N}$, and then projecting this gradient onto the tangent space of the manifold. Let $G \in \R^{p \times N}$ be the computed Euclidean gradient of $\ell$ at $V$, the projection of $G$ onto $T_V \mc M$ is decomposable into $N$ projections of the columns $g_i$ onto the tangent space of the individual sphere $\mc M_i$
\[
\Proj_{T_{v_{i}}\mc M_{i}}(g_i)=(I-\frac{1}{\alpha_{i}}v_{i}v_{i}^{\top})g_i \qquad \forall i.
\]
The next lemma formalizes the computation of the Riemannian gradient for $\ell$.
\begin{lemma}\label{lemma:Riemannian gradient}
    For any $V=[v_{1},\dots,v_{N}]\in \mc M$, the Riemannian gradient of $\ell$ in the $i$-th block is given by
    \begin{equation} \label{eq:r-grad}
        \grad_{i} \ell(V) =  g_i - \frac{1}{\alpha_i} g_i^\top v_i v_i \in \R^{p},
    \end{equation}
    where $g_i = -2 ((H + V) M^{-1})_i \in \R^{p}$ and $M =I+(H+V)^{\top}(H+V)$. Furthermore, the Riemannian gradient of $\ell$ at point $V$ corresponding to the product manifold $\mc M$ is given by
    $
        \grad \ell(V) = [\grad_{1} \ell(V),\dots,\grad_{N} \ell(V)]$.
\end{lemma}

\begin{algorithm*}[!ht]
\caption{Riemannian Block Coordinate Descent on Product of Spheres with Exponential Maps}\label{alg:RBCD with Exponential Maps}
\label{alg:algorithm1}
\begin{algorithmic}[1]
\REQUIRE Hidden activation vectors $\{h_i\}_{i=1}^N$, magnitudes $\{\alpha_i\}_{i=1}^N$, learning rate $\eta_i > 0$, max iterations $K$ 
\STATE Initialize $v_i^{(0)} $ using~\eqref{eq:baseline_steering}
\FOR{$k = 1$ to $K$}
    \FOR{$i = 1$ to $N$}
        \STATE Compute Euclidean gradient in the variable $v_i$ by $g_i \gets \nabla_{i} \ell(v_1^{(k)}, \dots, v_i^{(k-1)}, \ldots, v_N^{(k-1)})$
        \STATE Compute the descent direction $d_{i} \leftarrow - g_i + \frac{1}{\alpha_i} ((v_i^{(k-1)})^{\top} g_i)\, v_i^{(k-1)}$
        \STATE Compute $v_i^{(k)} \gets \cos( \frac{ \eta_i \| d_{i} \|_2 }{\sqrt{\alpha_{i}}}) v_i^{(k-1)} + \sin ( \frac{\eta_i \| d_{i} \|_2 }{\sqrt{\alpha_{i}}}) \frac{d_{i}}{\| d_{i} \|_2} \sqrt{\alpha_i}$
    \ENDFOR
\ENDFOR
\STATE \textbf{return} $V^{(K)} =[v_1^{(K)}, \dots, v_N^{(K)} ]$ 
\end{algorithmic}
\end{algorithm*}
The \textit{negative} Riemannian gradient at $V$ gives the direction of steepest descent within the tangent space. The exponential map then moves the incumbent solution along this direction, not in a straight line, but along a curved path that fits the spherical surface. This allows us to take meaningful steps while staying on the manifold throughout the iterations. Given the descent direction $d_i = - \grad_{i} \ell(V)$ and a step size $\eta_i$ for the $i$-th block, the exponential map gives~\citep[Section 10.2]{ref:boumal2023introduction}
\begin{align}
        \label{eq:exp_map}
       &\Exp_{v_i}( \eta_i d_{i} ) = \\
       &\quad \cos( \frac{\eta_i \|  d_{i} \|_2 }{\sqrt{\alpha_{i}}}) v_i + \sin ( \frac{\eta_i \| d_{i} \|_2 }{\sqrt{\alpha_{i}}}) \frac{d_{i}}{\| d_{i} \|_2} \sqrt{\alpha_i} \in \mc M_i. \notag
\end{align}
Given $\Exp_{v_i}( \eta_i d_{i} )$ in~\eqref{eq:exp_map}, the exponential map on the product manifold $\mc M$ is given by
\begin{align}\label{eq:exp_map_product}
       &\Exp_{V}( [\eta_{1}d_{1}, \dots, \eta_{N}d_{N}] ) = \\
       &\quad [\Exp_{v_1}( \eta_1 d_{1} ),\dots, \Exp_{v_N}( \eta_N d_{N} ) ] \in \mc M .\notag
\end{align}

\subsection{Block-Coordinate Riemannian Descent}

Since our manifold $\mc M$ decomposes as a product of scaled spheres, its natural block structure allows one global update to be replaced by $N$ cheaper spherical updates. Consequently, Riemannian Block-coordinate Descent~\cite{ref:gutman2023coordinate} provides an intuitive and efficient method for Problem~\eqref{eq:minlogdet_eq}.

In the outer iteration $k$, Algorithm~\ref{alg:algorithm1} updates the blocks $v_{i}$ sequentially for $i=1,\dots,N$. Fixing all other blocks at their most recent values, it first computes the Euclidean gradient $g_{i}$. In Line 5, $g_{i}$ is then projected onto the tangent space to yield the Riemannian descent direction $d_{i}$, and then Line 6 moves $v_{i}^{(k-1)}$ along the geodesic on $\mc M_{i}$ via the exponential map~\eqref{eq:exp_map}, thus preserving its feasibility on the sphere.

\textbf{Initialization.} To get a feasible initial point, for each $i$, we can initialize $v_i^{(0)}$ by
\begin{equation}
    v_i^{(0)} = \sqrt{\alpha_i} \frac{h_i + \varepsilon_i - \bar h}{ \| h_i + \varepsilon_i - \bar h \|_2} \qquad \forall i.
    \label{eq:baseline_steering}
\end{equation}
where $\varepsilon_i \sim \mc N(0, \sigma^2 I_d)$ are independent Gaussian noise with small variances and $\bar h = \frac{1}{N} \sum_{i=1}^N h_i$ is the centroid of the hidden vectors $\{h_i\}_{i=1}^N$. 

\subsection{Convergence Analysis}

We now study the convergence guarantee of Algorithm~\ref{alg:algorithm1}. The next theorem asserts the convergence of Algorithm~\ref{alg:algorithm1} for solving Problem~\eqref{eq:minlogdet_eq} with well-chosen step sizes. To this end, let $\bar{\alpha}\Let \sum_{i=1}^{N}\alpha_{i}$ be the total radii, and for each sequence $i$, define the quantity
\begin{equation}\label{eq:L_i}
        L_{i} \Let 2+4(\|H \|_{F}+\sqrt{\bar{\alpha}})^2 + \frac{2}{\sqrt{\alpha_{i}}} (\|H \|_{F}+\sqrt{\bar{\alpha}}) \quad \forall i.
    \end{equation}
    
\begin{theorem}[Convergence of Algorithm~\ref{alg:RBCD with Exponential Maps}]\label{thm:convergence}
 Let  $ V^{(k)} = [v^{(k)}_1, \ldots, v^{(k)}_N]$ be the sequence generated from Algorithm~\ref{alg:RBCD with Exponential Maps} with learning rate $\eta_{i} = 1/L_{i}$ for $i=1,\dots,N$. Define
\[
C \Let \frac{L_{\min}^{2}}{4L_{\max}(L_{\min}^{2}+L^{2}N(N-1))} ,
\]
where $L_{\min}=\min_{i}L_{i}$ and $L_{\max}=\max_{i}L_{i}$.
Then, the following hold:
    \begin{itemize}
        \item We have
        $
        \lim_{k\to \infty} \| \grad~\ell(V^{(k)}) \|_{F} = 0$ and
        \[
        \min_{s \leq k} \| \grad~\ell(V^{(s)}) \|_{F} \leq \sqrt{\frac{\ell(V^{(0)}) - \ell\opt}{C k}},
        \]
        where $\ell\opt$ is the optimal value of Problem~\eqref{eq:minlogdet_eq}.
        \item Any limit point $V^\infty \in \mc M$ of $\{V^{(k)}\}_{k\geq1}$ is a stationary point, i.e., $\grad~\ell(V^\infty)=0$.
    \end{itemize}
\end{theorem}
Theorem~\ref{thm:convergence} establishes both asymptotic and non-asymptotic convergence results for Algorithm~\ref{alg:algorithm1}. It guarantees that any limit point of the generated sequence is a stationary point of the objective function. The sublinear rate $O(1/\sqrt{k})$ for the minimum gradient norm bounds the number of iterations needed to achieve a desired accuracy level. 

To analyze the convergence of the algorithm, we study the smoothness of the objective function. Firstly, we recite the $ L$-smoothness definition in the Riemannian sense~\citep[equation (3)]{ref:gutman2023coordinate}.
\begin{definition}[$L$-smoothness]\label{def:L-smoothness}
        The function $\ell: \mathcal{M} \to \R$ is $L$-smooth if for all $V \in \mc M$, $U\in T_{V}\mc M$ and $Z=\Exp_{V}(U)$ it holds that
        \[
        \|\Gamma_{V\to Z}^{U}~\grad \ell(V)  -  \grad \ell(Z)\|_{F} \leq L \| U\|_{F},
        \]
        where $\Gamma_{V\to Z}^{U}: T_{V}\mc M \to T_{Z}\mc M$ is the parallel transport operator along the curve $\gamma (t)=\Exp_{V}(tU)$.
    \end{definition}

The following lemma shows that the objective function $\ell(V)$ of Problem~\eqref{eq:minlogdet_eq} satisfies Definition~\ref{def:L-smoothness} with an explicit Lipschitz constant $L$.

\begin{proposition}[Smoothness]\label{prop:L-smooth}
       Let $\alpha_{\min} \Let \min_{i}\alpha_{i} > 0$. The objective function $\ell(V)$ is $L$-smooth, with constant
        \begin{equation}
        L =  2 + 4(\|H\|_F + \sqrt{\bar{\alpha}} )^{2} + \frac{2}{\sqrt{\alpha_{\min}}}( \| H\|_{F}+\sqrt{\bar{\alpha}} ).
        \end{equation}
        
\end{proposition}

We further show that the function $\ell$ also satisfies the block smoothness, which is defined by restricting Definition~\ref{def:L-smoothness} to each individual sphere $\mc M_{i}$. 

\begin{proposition}[Block smoothness]\label{prop:block_L-smooth}
    Let $V=[v_{1},\dots,v_{N}] \in \mc M$, $U=[u_{1},\dots,u_{N}]\in T_{V}\mc M$ and $Z=\Exp_{V}(U)$. For any $v_{i}$ and $u_{i} \in T_{v_{i}}\mc M_{i}$, it holds that
    \begin{equation}\label{eq:block_L-smooth}
         \|\Gamma_{v_{i}\to z_{i}}^{u_{i}}~\grad_{i} \ell(V)  -  \grad_{i} \ell(Z)\|_{F} \leq L_{i} \| u_{i}\|_{2}
    \end{equation}
    with $L_{i}$ defined in \eqref{eq:L_i}. Here $\Gamma_{v_{i}\to z_{i}}^{u_{i}}:T_{v_{i}}\mc M_{i} \to T_{z_{i}}\mc M_{i}$ is the parallel transport operator along the curve $\gamma(t)=\Exp_{v_{i}}(t u_{i})$.
\end{proposition}
Proposition \ref{prop:block_L-smooth} is crucial both for establishing the convergence of Algorithm~\ref{alg:algorithm1} and for guiding the selection of its step‐sizes. Equation~\eqref{eq:block_L-smooth} together with~\citet[Corollary 10.54]{ref:boumal2023introduction} implies that for each $i$ it holds
    \begin{align}
        & \ell(v_{1},\dots,\Exp_{v_{i}}(u_{i}),\dots,v_{i}) \leq \label{eq:block_L-smooth2} \\
        & \hspace{2cm} \ell(V) + \langle \grad_{i}~\ell(V),u_{i} \rangle + \frac{L_{i}}{2} \|u_{i} \|_{2}^{2}.\notag
    \end{align}
Thus, in particular, if one sets $\eta_{i}=1/L_{i}$ then each coordinate‐descent update yields a provable decrease in the objective~\cite[Lemma~1]{ref:gutman2023coordinate}. This obviates the need for extensive tuning of learning rates $\eta_{i}$.

\begin{table*}[]
\begin{tabular}{l|c|ccc|ccc|ccc}
\hline
\multirow{3}{*}{Model}                                                                 & \multirow{3}{*}{Temp.} & \multicolumn{3}{c|}{AIME24}                               & \multicolumn{3}{c|}{MATH500}                              & \multicolumn{3}{c}{OlympiadBench}                         \\ \cline{3-11} 
&                                       & \multicolumn{2}{c}{SPREAD}    & \multirow{2}{*}{Sampling} & \multicolumn{2}{c}{SPREAD}    & \multirow{2}{*}{Sampling} & \multicolumn{2}{c}{SPREAD}    & \multirow{2}{*}{Sampling} \\ \cline{3-4} \cline{6-7} \cline{9-10}
&                                       & $C = 1$         & $C = 10$        &                           & $C = 1$         & $C = 10$        &                           & $C = 1$         & $C = 10$        &                           \\ \hline
\multirow{5}{*}{Qwen2.5-1.5B}                                                          & 1.0                                   & 3.3           & \textbf{6.7}  & 0.0                       & 43.2          & \textbf{43.4} & 42.8                      & \textbf{21.5} & 19.7          & 19.0                      \\
& 0.8                                   & \textbf{6.7}  & \textbf{6.7}  & 3.3                       & 51.8          & \textbf{54.2} & 51.4                      & 25.3          & \textbf{27.0} & 25.9                      \\
& 0.6                                   & \textbf{10.0} & 3.3           & 3.3                       & \textbf{55.0} & 54.0          & 53.4                      & 28.4          & 28.3          & \textbf{29.3}             \\
& 0.4                                   & 6.7           & 6.7           & 6.7                       & 56.2          & \textbf{57.6} & 55.8                      & \textbf{30.8} & 30.1          & 30.7                      \\
& 0.2                                   & \textbf{10.0} & 6.7           & 6.7                       & \textbf{55.6} & 53.8          & 52.2                      & \textbf{28.0} & 27.1          & 26.4                      \\ \hline
\multirow{5}{*}{\begin{tabular}[c]{@{}l@{}}Qwen2.5-Math\\ -1.5B-Instruct\end{tabular}} & 1.0                                   & \textbf{26.7} & 16.7          & 20.0                      & 83.8          & 83.6          & \textbf{84.6}             & 47.6          & \textbf{49.5} & 48.3                      \\
& 0.8                                   & \textbf{26.7} & 23.3          & 16.7                      & \textbf{85.2} & 85.0          & 83.6                      & 49.9          & 49.2          & \textbf{51.0}             \\
& 0.6                                   & \textbf{26.7} & \textbf{26.7} & 20.0                      & \textbf{85.4} & 84.4          & 84.6                      & 50.4          & \textbf{51.0} & 50.8                      \\
& 0.4                                   & 23.3          & 23.3          & 23.3                      & \textbf{84.6} & 84.2          & 84.0                      & \textbf{51.7} & 48.8          & 51.0                      \\
& 0.2                                   & 20.0          & \textbf{26.7} & 16.7                      & 82.4          & 84.4          & 84.4                      & \textbf{52.3} & 51.1          & 49.9                     
\end{tabular}
\caption{Pass@$N$ $\uparrow$ (in \%) performance comparison across model variants and sampling temperature on three mathematical reasoning benchmarks. Bold indicates the best method in each dataset.}
\label{fig:results-pass@n}
\end{table*}

\section{Numerical Experiment}\label{sec:experiment}
We evaluate SPREAD across three established mathematical reasoning benchmarks: AIME24 (30 problems),  MATH500 (500 problems), and OlympiadBench (675 problems). We evaluate using the following metrics:
\begin{itemize}
    \item \textit{Pass@$N$}: The proportion of problems for which at least one of the $N$ generated solutions produces the correct final answer.
    \item \textit{Solution Diversity}: The float score indicating the overall diversity among $N$ solution
    \item \textit{Unique Solution Count}: The number of distinct solution approaches among $N$ solutions.
\end{itemize}
For the latter two metrics, we employ a language-model-as-a-judge paradigm using GPT-4.1-mini. We prompt the model with the question and $N$ generated solutions, requesting a diversity score (float in $[0,1]$) and a count of unique approaches (integer). The specific prompts and parameters for this evaluation are detailed in the Appendix.
Our experiments utilize two model variants: Qwen2.5-1.5B (base) and Qwen2.5-Math-1.5B-Instruct (math-specialized). Steering vectors are applied at layer 28, which corresponds to the final layer in both architectures. The hidden activations have dimension $p = 1536$. The magnitude $\alpha_i$ is chosen as $\alpha_i = C \|h_i\|_2 / p$, where $C \in \{1, 10\}$ is a scaling constant.
Steering vectors are computed using Algorithm~\ref{alg:algorithm1} at token positions $\tau \in \{100, 600, 1100, 1600\}$ with $K = 20$, and the learning rates are calibrated using the input activations following Theorem~\ref{thm:convergence}. 
All experiments are conducted on NVIDIA RTX A5000 24GB hardware using temperature sampling for solution generation with temperature $ \in \{0.2, 0.4, 0.6, 0.8, 1.0 \}$ and a maximum generation length of 2048 tokens. To ensure reproducibility, we fix the random seed to 42 for all experimental runs. Additional numerical results are presented in the appendix. 

\subsection{Experiment Results}

Table~\ref{fig:results-pass@n} summarizes the performance of SPREAD compared to temperature sampling across various benchmarks and temperature settings. SPREAD with either $C=1$ or $C=10$ consistently performs at least as well as vanilla temperature sampling, and often achieves improvements of several percentage points. Overall, SPREAD decoding (especially with $C=1$) offers the strongest results under most conditions. 
Table~\ref{fig:results-uniquesolution} demonstrates that SPREAD consistently outperforms temperature sampling in generating unique solutions across multiple benchmarks. These results highlight SPREAD’s robustness and effectiveness in promoting correct and diverse reasoning trajectories across various model variants and datasets.

\begin{table*}[]
\begin{tabular}{c|c|c|ccc|cccccc}
\hline
\multirow{3}{*}{Metric}                                 & \multirow{3}{*}{Model}                                                                       & \multirow{3}{*}{Temp.} & \multicolumn{3}{c|}{AIME24}                               & \multicolumn{3}{c}{MATH500}                                                    & \multicolumn{3}{c}{OlympiadBench}                         \\ \cline{4-12} 
&                                                                                              &                        & \multicolumn{2}{c}{SPREAD}    & \multirow{2}{*}{Sampling} & \multicolumn{2}{c}{SPREAD}    & \multicolumn{1}{c|}{\multirow{2}{*}{Sampling}} & \multicolumn{2}{c}{SPREAD}    & \multirow{2}{*}{Sampling} \\ \cline{4-5} \cline{7-8} \cline{10-11}
&                                                                                              &                        & $C = 1$       & $C = 10$      &                           & $C = 1$       & $C = 10$      & \multicolumn{1}{c|}{}                          & $C = 1$       & $C = 10$      &                           \\ \hline
\multirow{10}{*}{\rotatebox{90}{Unique Solution Count $\uparrow$}} & \multirow{5}{*}{\begin{tabular}[c]{@{}c@{}}Qwen2.5\\ -1.5B\end{tabular}}                     & 1.0                    & \textbf{6.97} & 6.60          & 6.67                      & 3.14          & \textbf{3.21} & \multicolumn{1}{c|}{3.14}                      & 6.12          & 6.14          & 6.14                      \\
&                                                                                              & 0.8                    & \textbf{6.83} & 6.43          & 3.6                       & 3.03          & \textbf{3.05} & \multicolumn{1}{c|}{2.97}                      & 5.99          & 5.99          & \textbf{6.04}             \\
&                                                                                              & 0.6                    & \textbf{6.40} & 6.07          & 3.37                      & \textbf{2.91} & 2.86          & \multicolumn{1}{c|}{2.90}                      & 5.57          & 5.65          & \textbf{5.58}             \\
&                                                                                              & 0.4                    & 6.00          & \textbf{6.20} & 6.17                      & 2.73          & 2.74          & \multicolumn{1}{c|}{\textbf{2.77}}             & \textbf{5.30} & 5.30          & 5.26                      \\
&                                                                                              & 0.2                    & 5.17          & \textbf{5.60} & 2.97                      & 2.47          & 2.56          & \multicolumn{1}{c|}{\textbf{2.59}}             & 4.65          & 4.69          & 4.69                      \\ \cline{2-12} 
& \multirow{5}{*}{\begin{tabular}[c]{@{}c@{}}Qwen2.5\\ -Math\\ -1.5B\\ -Instruct\end{tabular}} & 1.0                    & 6.63          & \textbf{7.03} & 3.67                      & 1.92          & \textbf{1.94} & \multicolumn{1}{c|}{1.93}                      & \textbf{4.59} & 4.53          & 4.57                      \\
&                                                                                              & 0.8                    & 6.63          & \textbf{6.73} & 3.47                      & 1.87          & \textbf{1.90} & \multicolumn{1}{c|}{1.89}                      & 4.28          & 4.36          & \textbf{4.39}             \\
&                                                                                              & 0.6                    & 6.27          & \textbf{6.47} & 3.5                       & 1.81          & \textbf{1.84} & \multicolumn{1}{c|}{1.82}                      & 4.10          & \textbf{4.16} & 4.11                      \\
&                                                                                              & 0.4                    & \textbf{6.23} & 5.87          & 6.17                      & 1.77          & 1.73          & \multicolumn{1}{c|}{\textbf{1.78}}             & \textbf{3.99} & 3.92          & 3.89                      \\
&                                                                                              & 0.2                    & \textbf{5.87} & \textbf{5.87} & 5.53                      & 1.70          & 1.74          & \multicolumn{1}{c|}{1.74}                      & 3.67          & 3.65          & \textbf{3.68}             \\ \hline
\multirow{10}{*}{\rotatebox{90}{Diversity Score $\uparrow$}}       & \multirow{5}{*}{\begin{tabular}[c]{@{}c@{}}Qwen2.5\\ -1.5B\end{tabular}}                     & 1.0                    & \textbf{0.37} & 0.33          & 0.40                      & \textbf{0.40} & 0.39          & \multicolumn{1}{c|}{0.39}                      & 0.38          & \textbf{0.40} & 0.39                      \\
&                                                                                              & 0.8                    & \textbf{0.51} & 0.45          & 0.37                      & 0.42          & \textbf{0.44} & \multicolumn{1}{c|}{0.41}                      & 0.46          & 0.46          & 0.46                      \\
&                                                                                              & 0.6                    & \textbf{0.52} & 0.46          & 0.39                      & \textbf{0.41} & 0.40          & \multicolumn{1}{c|}{0.40}                      & 0.46          & 0.46          & 0.46                      \\
&                                                                                              & 0.4                    & 0.41          & \textbf{0.47} & 0.42                      & 0.36          & 0.36          & \multicolumn{1}{c|}{0.36}                      & 0.42          & \textbf{0.43} & 0.42                      \\
&                                                                                              & 0.2                    & 0.31          & \textbf{0.34} & 0.33                      & \textbf{0.29} & \textbf{0.29} & \multicolumn{1}{c|}{0.28}                      & 0.33          & \textbf{0.35} & \textbf{0.35}             \\ \cline{2-12} 
& \multirow{5}{*}{\begin{tabular}[c]{@{}c@{}}Qwen2.5\\ -Math\\ -1.5B-\\ Instruct\end{tabular}} & 1.0                    & 0.64          & \textbf{0.65} & 0.63                      & 0.25          & 0.25          & \multicolumn{1}{c|}{0.25}                      & \textbf{0.39} & \textbf{0.39} & 0.38                      \\
&                                                                                              & 0.8                    & 0.60          & \textbf{0.62} & 0.57                      & 0.23          & 0.23          & \multicolumn{1}{c|}{0.23}                      & 0.36          & \textbf{0.38} & 0.37                      \\
&                                                                                              & 0.6                    & \textbf{0.57} & 0.55          & 0.55                      & 0.21          & \textbf{0.22} & \multicolumn{1}{c|}{0.20}                      & 0.34          & 0.34          & 0.34                      \\
&                                                                                              & 0.4                    & 0.51          & 0.51          & 0.51                      & \textbf{0.19} & 0.18          & \multicolumn{1}{c|}{0.18}                      & \textbf{0.31} & 0.30          & 0.31                      \\
&                                                                                              & 0.2                    & 0.46          & \textbf{0.52} & 0.41                      & 0.15          & \textbf{0.16} & \multicolumn{1}{c|}{0.15}                      & 0.26          & 0.26          & 0.26                     
\end{tabular}
\caption{Unique Solution Count $\uparrow$ and Diversity Score $\uparrow$ comparison across model variants and sampling temperature on three mathematical reasoning benchmarks. Bold indicates the best method in each dataset.}
\label{fig:results-uniquesolution}
\end{table*}

\begin{figure}[htbp]
    \centering
    \includegraphics[width=0.43\textwidth]{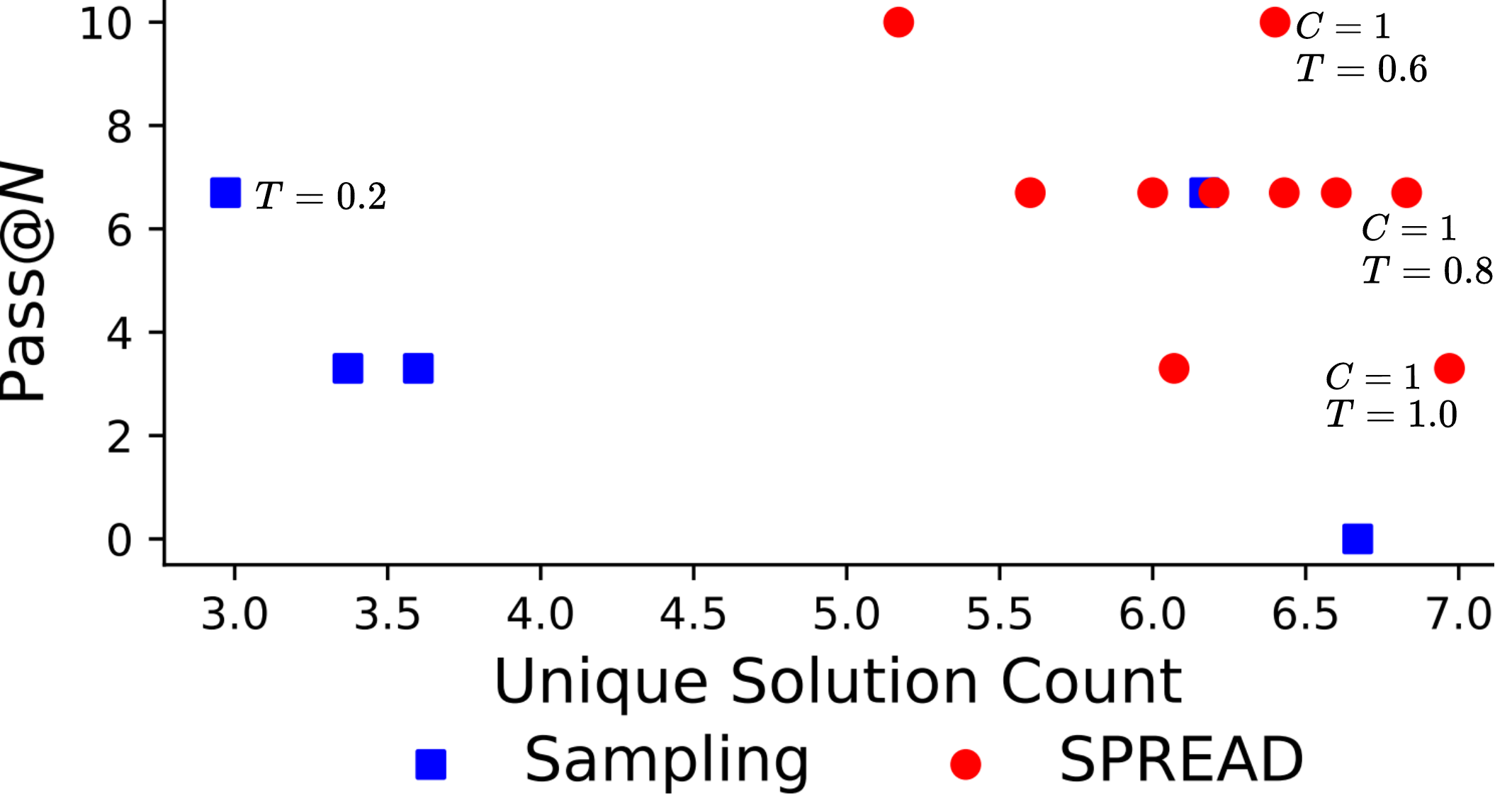}
    \caption{Comparison of SPREAD and sampling methods on AIME24 dataset using Qwen2.5-1.5B model, showing Pass@$N$ and Unique Solution Count.}
    \label{fig:comparison-aime24}
\end{figure}
\begin{figure}[htbp]
    \centering
    \includegraphics[width=0.43\textwidth]{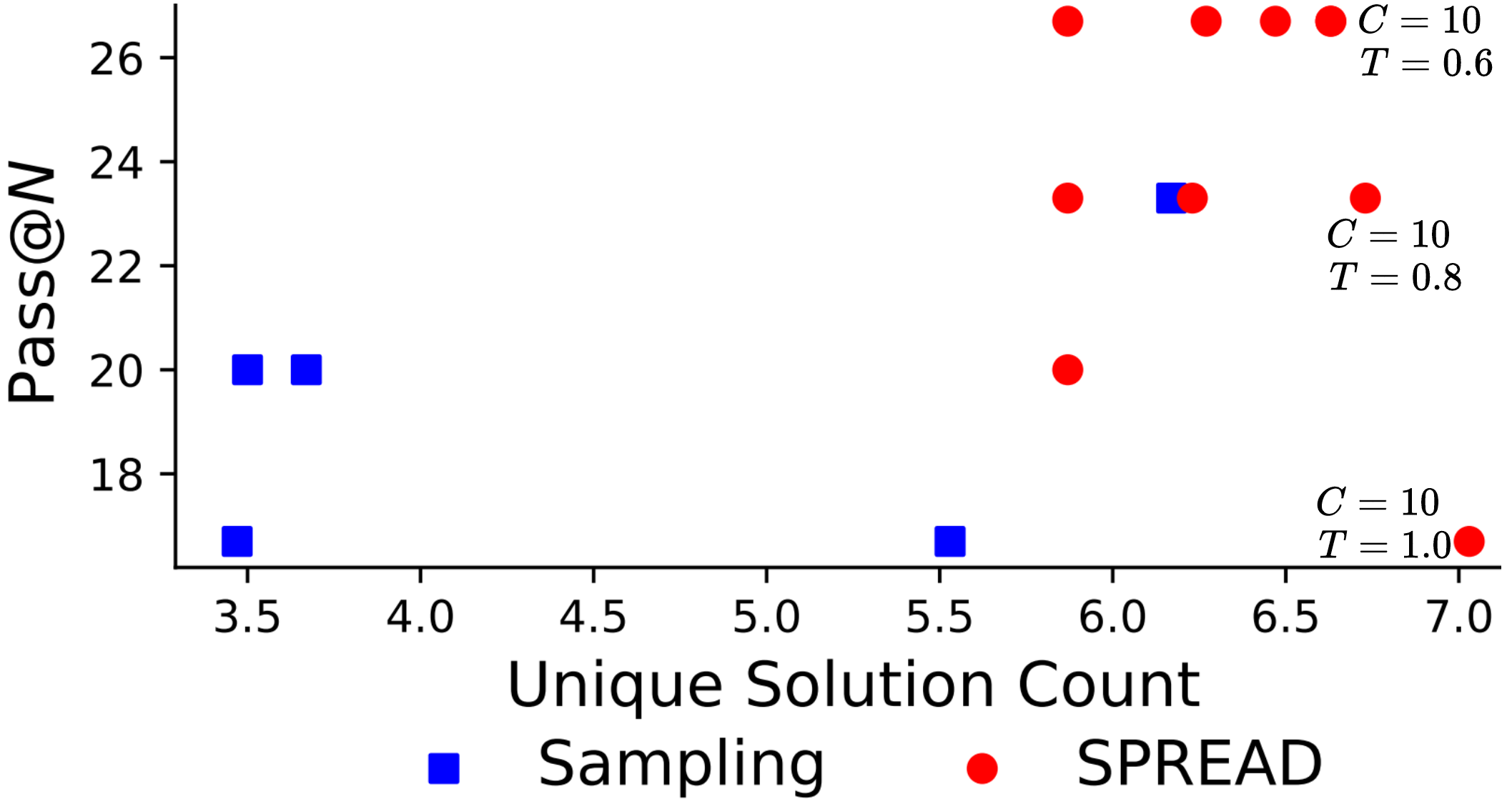}
    \caption{Comparison of SPREAD and sampling methods on AIME24 dataset using Qwen2.5-Math-1.5B-Instruct model, showing Pass@$N$ and Unique Solution Count.}
    \label{fig:comparison-math}
\end{figure}
To strengthen the comparison, we aggregate the metrics and analyze the accuracy-diversity frontiers in Figure \ref{fig:comparison-aime24} and~\ref{fig:comparison-math} for the AIME24 dataset. We can observe that the performance of SPREAD (red circles) dominates that of vanilla sampling methods with different temperatures (blue squares). The Pareto plots for the MATH and OlympiadBench datasets are presented in the appendix. 

\subsection{Computational Efficiency}

\begin{figure}[htbp]
    \centering
    \includegraphics[width=0.46\textwidth]{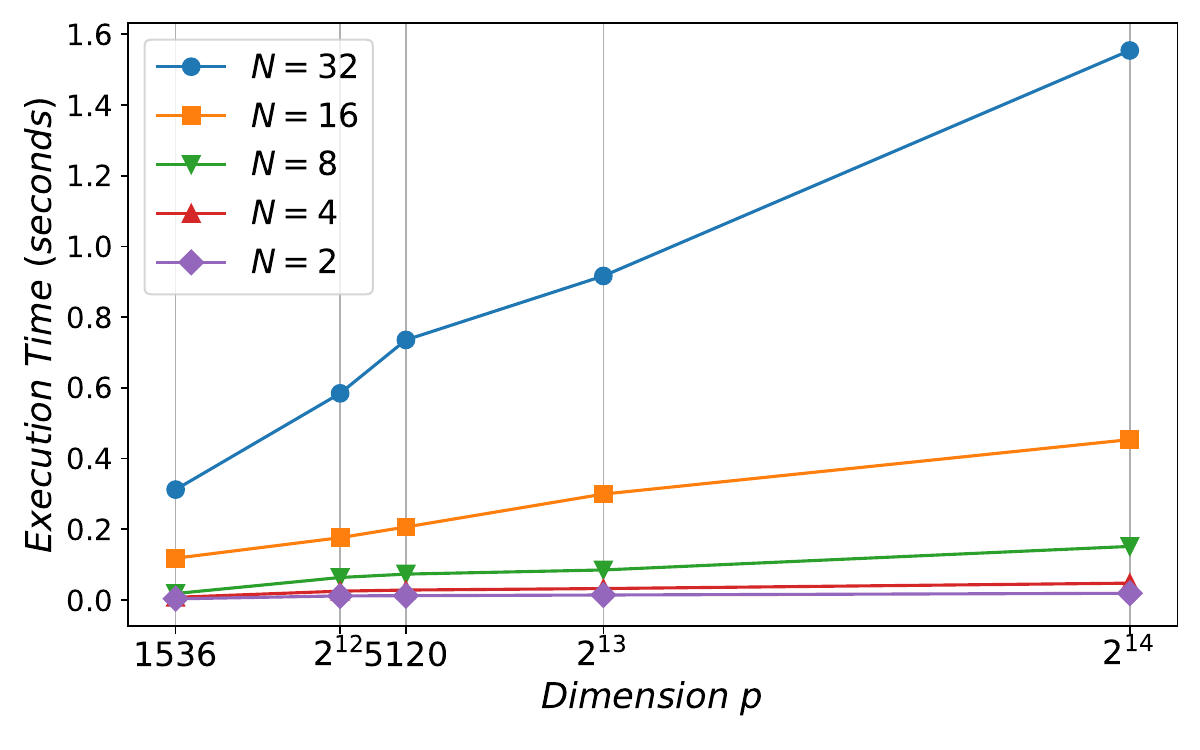}
    \caption{Average running time of Algorithm 1 with varying problem dimension $p$ and number of steering vectors $N$. Lower execution time is better.}
    \label{fig:runtime-dimensions}
    \vspace{-5mm}
\end{figure}

We analyze the execution time of Algorithm~\ref{alg:algorithm1} for different hidden vector dimensions $p$ and different numbers of sequences $N$. We generate synthetic activations with $p$ varying from $1536$ up to $2^{14} = 16384$, which corresponds to the hidden size of large-scale models such as LLaMA-3.1-405B. Figure~\ref{fig:runtime-dimensions} reports the average solution time of Algorithm 1 with $K = 20$ over 30 independent runs. The execution time for $N = 32$ remains under 1.8 seconds even for the largest dimension $16384$, showing that our SPREAD method remains highly efficient. These results demonstrate the scalability and practical efficiency of our algorithm for inference-time language model intervention.

\section{Conclusions}\label{sec:conclusion}
We introduced SPREAD, a novel unsupervised activation steering framework at test time designed to improve the reasoning diversity of language models. By addressing the well-known challenge of low output variance in best-of-$N$ sampling, SPREAD intervenes meaningful diversity into generation trajectories. Our key innovation lies in reformulating the activation steering problem as a Riemannian optimization over the product of spheres, maximizing the log-determinant volume spanned by intervened representations. This principled formulation enables us to efficiently compute diverse steering directions using Riemannian block-coordinate descent. Empirical evaluations on mathematical reasoning benchmarks such as AIME24, MATH, and OlympiadBench validate the effectiveness of our approach, as SPREAD significantly improves both the diversity and accuracy of generated solutions, outperforming temperature-sampling techniques. These results demonstrate the potential of geometric test-time intervention methods to enhance reasoning in language models.

\bibliography{arxiv.bbl}

\newpage
\appendix
\onecolumn

This is the appendix for the paper ``Test-time Diverse Reasoning by Riemannian Activation Steering". Appendix~\ref{app:proof} collects the proofs of all theoretical results in the main paper.  Appendix~\ref{app:Experimental Details} provides further details about the experimental settings and additional results.

\section{Proofs of Main Results}\label{app:proof}
This section contains the proofs of all technical results presented in the main paper. For a matrix $A\in \R^{p\times N}$, $\|A\|_{2}$ denotes its spectral norm, i.e., the square root of the largest eigenvalue of the matrix $A^{\top}A$.

\begin{proof}[Proof of Proposition~\ref{prop:obj}]
The volume of the parallelepiped is computed from the determinant of the Gram matrix~:
\[
    \mathrm{Volume}( \mc P (\{ h_i + v_i\}_{i \in \mathbb I} ) )^2 = \det ( (H+V)_{\mathbb I}^\top (H+V)_{\mathbb I} )
\]
where $(H+V)_{\mathbb I}$ is the submatrix of $(H+V)$ restricted to the columns indexed by $\mathbb I$~\cite{ref:pyle1962non}. By~\citet[theorem~2.1]{ref:kulesza2012determinantal}, we find
\[
    \sum_{\mathbb I \subseteq 2^{[N]}}~\det ( (H+V)_{\mathbb I}^\top (H+V)_{\mathbb I} ) = \det (I + (H + V)^\top (H + V))
\]
Consequently, problem~\eqref{eq:volume} is equivalent to 
\begin{equation} \label{eq:maxdet}
    \max_{\forall i: \| v_i \|_2^2 \le \alpha_i}~\det  [ I +   (H+V)^\top (H + V) ].
\end{equation}
Since the logarithm function is a monotonically increasing function, we obtain the desired result. 
\end{proof}

    \begin{proof}[Proof of Proposition~\ref{prop:equality_refor}]
        Introducing dual variables $\lambda_{i}\geq0$ for each of the $N$ constraints of problem~\eqref{eq:minlogdet}, then the Lagrangian function for problem~\eqref{eq:minlogdet} is
        \[
        \mathcal{L}(V, \lambda) = \ell(V) + \sum_{i}\lambda_{i}(\| v_i \|_2^2 - \alpha_i ).
        \]
        For each $i=1,\dots,N$, we define the constraint function $c_{i}(V) = \|v_{i} \|^{2}-\alpha_{i}$. We have the gradient \begin{equation}\label{eq:grad_cons}
        \nabla_{i} c_i(V)= (0,\ldots,0,\underbrace{2v_i}_{i\mathrm{-th~block}},0,\ldots,0).
        \end{equation}
        Let  $V\opt=[v_{1}\opt, \dots, v_{N}\opt] $ be an optimal solution of problem~\eqref{eq:minlogdet}. Consider the active set at $V\opt$ 
        \[
        \mathcal{A} = \{i: c_{i}(V\opt)=\alpha_{i}\} = \{i:\| v_{i}\|^{2} = \alpha_{i}\}.
        \]
        Since $\alpha_{i}>0$, from \eqref{eq:grad_cons} we know that $\{ \nabla c_{i}(V\opt), i \in \mathcal{A}\}$ are linearly independent. Thus, the KKT point $(V\opt, \lambda\opt)$ of problem~\eqref{eq:minlogdet} satisfies
        \begin{subequations}
            \begin{align}
            \nabla_{i} \mathcal{L}(V\opt, \lambda\opt) & = 0 \quad \forall i \label{eq:kkt_sta}\\
            \| v_{i}\opt \|_2^2 -\alpha_i \le 0, \quad \lambda_{i}\opt &\geq 0 \quad \forall i \\
            \lambda_{i}\opt (\| v_{i}\opt \|_2^2 -\alpha_i ) & = 0 \quad \forall i. \label{eq:kkt_slack}
        \end{align}
        \end{subequations}
      The gradient of $\ell(V) $ with respect to each vector $v_i$ is computed as 
        \[
        \nabla_{i} \ell(V) = \frac{\dd \ell}{\dd V} e_{i} = -2 (H + V)[I+(H + V)^{\top}(H + V)]^{-1} e_{i}.
        \]
        Thus, the gradient of the Lagrangian with respect to each vector $v_i$ is
        \[
        \nabla_{i} \mathcal{L}(V\opt, \lambda\opt) = -2 PQ e_{i} - 2\lambda_{i}\opt v_{i},
        \]
        where $P=H + V\opt$ and $Q = (I+P^{\top}P)^{-1}$.
        
       In the following, we prove that any KKT point should satisfy $\lambda\opt > 0$, and hence by~\eqref{eq:kkt_slack}, it implies that $\| v_{i}\opt \|_2^2 -\alpha_i = 0$ for all $i$. To this end, suppose without any loss of generality that $\lambda_{1}\opt = 0$. Then the stationary condition~\eqref{eq:kkt_sta} gives
        \[
        PQe_{1}=0,
        \]
        which means that
        \begin{equation}\label{eq:Pq1}
             Pq_{1} = 0,
        \end{equation}
        where $q_{1} $ is the first column of $Q$. From the definition of the inverse matrix, it holds that
        \[
        (I+P^{\top}P) Q= I.
        \]
        Equating the first column of both sides gives
        \begin{equation}\label{eq:Pq1_2}
            (I+P^{\top}P)q_{1}=e_{1}.
        \end{equation}
        Equation~\eqref{eq:Pq1} implies that $P^\top P q_1 = 0$, thus, we have $q_{1}=e_{1}$. Consequentially, equation~\eqref{eq:Pq1} implies that $Pe_{1}=0$, which means that the first column of $P=H+V\opt$ is $0$, that is, $h_{1}+v_{1}\opt=0$. Now we construct a new matrix $\tilde{V}$ as $\tilde{V} = [\tilde{v}, v_{2}\opt, \dots,v_{N}\opt]$, where $\tilde{v}$ is any vector with norm $\sqrt{\alpha_{1}}$ but not equal to $-h_{1}$. One can see that $ \tilde{V}$ is feasible for problem~\eqref{eq:minlogdet}. Using the Sylvester's determinant identity $\det(I+AB)=\det(I+BA)$, we have
        \begin{align*}
            \ell(\tilde{V}) & =-\log\det[I + (H+\tilde{V})^{\top}(H+\tilde{V})] \\
            & = -\log\det[I + (H+\tilde{V})(H+\tilde{V})^{\top}] \\
            & = -\log\det[I + \sum_{i=2}^{N}(h_{i}+v_{i}\opt)(h_{i}+v_{i}\opt)^{\top} + (h_{1}+\tilde{v})(h_{1}+\tilde{v})^{\top}]\\
            & = -\log\det[I + PP^{\top} + (h_{1}+\tilde{v})(h_{1}+\tilde{v})^{\top}]\quad \quad (h_{1}+v_{1}\opt=0)\\
            & = -\log\det(I + PP^{\top}) -\log\det[I +(I + PP^{\top})^{-1}(h_{1}+\tilde{v})(h_{1}+\tilde{v})^{\top}] \\
            & = \ell(V\opt) -\log[1+(h_{1}+\tilde{v})^{\top}(I + PP^{\top})^{-1}(h_{1}+\tilde{v})].
        \end{align*}
        Since $I+PP^{\top}$ is positive definite and $\tilde{v}\neq -h_{1}$, we have $(h_{1}+\tilde{v})^{\top}(I + PP^{\top})^{-1}(h_{1}+\tilde{v})>0$, which implies that
        \[
        \ell(\tilde{V}) = \ell(V\opt) -\log[1+(h_{1}+\tilde{v})^{\top}(I + PP^{\top})^{-1}(h_{1}+\tilde{v})] < \ell(V\opt) .
        \]
        Hence, we get a contradiction about the primal-dual optimality of $(V\opt, \lambda\opt)$. Thus we can conclude that $\lambda\opt_{i}>0$ for all $i=1,\dots,N$. Combined with condition~\eqref{eq:kkt_slack}, it holds that $\|v_{i} \opt\|^{2}=\alpha_{i}$ for all $i$. The proof is completed.
    \end{proof}

\begin{proof}[Proof of Lemma~\ref{lemma:Riemannian gradient}]
The Euclidean gradient of $\ell$ is given by~\cite{ref:petersen2008matrix}
\begin{equation}\label{eq:e-grad_V}
    \nabla \ell(V) = -2 (H + V)M^{-1} \in \R^{p \times N} ,
\end{equation}
where $M =I+(H+V)^{\top}(H+V) \in \R^{N \times N}$. For a column vector $v_i$, the gradient of $\ell$ in this vector is the $k$-th column of the matrix $-2 (H + V)M^{-1}$, i.e., 
\begin{equation}\label{eq:e-grad}
    g_i \Let \nabla_{i} \ell(V) = -2 ((H + V) M^{-1})_i \in \R^{p}.
\end{equation}
The tangent space of the scaled sphere $\mc M_{i}$ is given by
 \[
        T_{v_i} \mc M_i = \{ z \in \R^{p} : z^\top v_i = 0 \},
\]
whose orthogonal projection is 
\[
\Proj_{T_{v_{i}}\mc M_{i}}(u)=(I-\frac{1}{\alpha_{i}}v_{i}v_{i}^{\top})u, \quad u \in \R^{p}.
\]
Therefore, the Riemannian gradient of $\ell$ at any point $v_{i} \in \mc M_{i}$ is given by
\begin{equation} 
        \grad_{i} \ell(V) =\Proj_{T_{v_{i}}\mc M_{i}}(g_{i})  = g_i - \frac{1}{\alpha_i} g_i^\top v_i v_i \in \R^{p}.
    \end{equation}
Since $\mc M$ is a product manifold, 
according to \citet[Exercise 3.67]{ref:boumal2023introduction}, the Riemannian gradient of $\ell$ at any point $V \in \mc M$ is 
\[
\grad \ell(V) = [\grad_{1} \ell(V),\dots,\grad_{N} \ell(V)].
\]
This completes the proof.
\end{proof}

The proof of Proposition~\ref{prop:L-smooth} relies on the following Lemma~\ref{lem:bound_hessian} and Proposition~\ref{prop:formula_Hess}.

 \begin{lemma}[Upper bound of Hessian]\label{lem:bound_hessian}
    Let $\bar{\alpha} \Let \sum_{i=1}^{N} \alpha_i$. For any $V \in \mc M$, it holds that
    \[
   \sup_{\|U\|_F=1} \|\nabla^2 \ell(V)[U]\|_F  \leq 2+4 \left(\|H \|_{F} + \sqrt{\bar{\alpha}}\right)^{2}.
    \]
\end{lemma}
\begin{proof}[Proof of Lemma~\ref{lem:bound_hessian}]
For any $V \in \mc M$ and $U \in \R^{p \times N}$ with $\|U\|_{F}=1$, the Hessian of $\ell$ at $V$ evaluated in the direction $U$ is~\cite{ref:petersen2008matrix}
\begin{equation*}
    \nabla^2 \ell(V)[U] = -2U M^{-1} + 2(H+V) M^{-1} (U^{\top}(H+V) + (H+V)^{\top} U) M^{-1},
\end{equation*}
where $M = I + (H + V)^\top (H + V)$. 
Thus we have
\begin{equation}\label{eq:hessian_bound_1}
  \|  \nabla^2 \ell(V)[U] \|_{F}  \leq  2 \| U M^{-1}\|_{F} + 2\|(H+V) M^{-1} (U^{\top}(H+V) + (H+V)^{\top} U) M^{-1} \|_{F}.
\end{equation}

Since $M=I + (H+V)^{\top}(H+V)$ and $(H+V)^{\top}(H+V)$ is always positive semidefinite, the smallest eigenvalue of $M$ must be equal or greater than $1$. Thus, the largest eigenvalue of $M^{-1}$ must be no greater than $1$, which means $\| M^{-1}\|_{2}\leq 1$. Then, we have
\begin{equation}\label{eq:upper-1}
    \|UM^{-1} \|_{F} \leq   \|U \|_{F}\| M^{-1}\|_{2}\leq \|U \|_{F}=1.
\end{equation}

For the term $\|(H+V) M^{-1} (U^{\top}(H+V) + (H+V)^{\top} U) M^{-1}\|_{F}$, it holds that
\begin{align} \notag & \|(H+V) M^{-1} (U^{\top}(H+V) + (H+V)^{\top} U) M^{-1}\|_{F} \\
  \nonumber   \leq &\|H+V \|_{F} \|M^{-1} \|_{2} \| U^{\top}(H+V) + (H+V)^{\top} U\|_{2} \| M^{-1}\|_{2} \\
  \nonumber   \leq & \|H+V \|_{F} \| U^{\top}(H+V) + (H+V)^{\top} U\|_{2} \\
  \nonumber   \leq & \|H+V \|_{F} \left( \| U^{\top}(H+V)\|_{2}+ \| (H+V)^{\top} U\|_{2} \right) \\
  \nonumber   \leq & \|H+V \|_{F} \times 2 \| U\|_{2} \|H+V \|_{2}  \\
    \leq & 2 \|H+V \|_{F}^{2} & (\| \cdot\|_{2}\leq \|\cdot \|_{F}).\label{eq:hessian_bound_2}
\end{align}
Defining $\bar{\alpha}=\sum_i \alpha_i$, according to the triangle inequality, we have
\[
\|H+V \|_{F} \leq \|H\|_{F} + \|V\|_{F} = \| H\|_{F} + \sqrt{\bar{\alpha}},
\]
which, together with~\eqref{eq:hessian_bound_2}, implies that
\begin{equation}\label{eq:hessian_bound_3}
   \|(H+V) M^{-1} (U^{\top}(H+V) + (H+V)^{\top} U) M^{-1}\|_{F} \leq  2(\| H\|_{F} + \sqrt{\bar{\alpha}})^{2}.
\end{equation}
Combining \eqref{eq:hessian_bound_1}, \eqref{eq:upper-1} and \eqref{eq:hessian_bound_3}, we have
\[
 \|  \nabla^2 \ell(V)[U] \|_{F}  \leq 2+4(\| H\|_{F} + \sqrt{\bar{\alpha}})^{2},
\]
which holds for all $U\in \R^{p \times N}$ with $\|U \|_{F}=1$. This completes the proof.
\end{proof}

\begin{proposition}
    [Formula  for $\Hess~\ell(V)$]\label{prop:formula_Hess}
    For all $V \in \mc M$ and $U \in T_{V}\mc M$, it holds that
    \begin{equation}
        \Hess~\ell(V)[U] = \Proj_{T_{V}\mc M}(\nabla^{2}\ell(V)[U] )+ U\Lambda,
    \end{equation}
where $\Lambda \in \R^{N \times N}$ is a diagonal matrix with entries $\Lambda_{ii} =- \frac{1}{\alpha_{i}}g_{i}^{\top}v_{i}$.
\end{proposition}
\begin{proof}[Proof of Proposition~\ref{prop:formula_Hess}]
   \citet[Corollary 5.16]{ref:boumal2023introduction} shows that
    \begin{equation}
        \Hess~\ell(V)[U] = \Proj_{T_{V}\mc M}(\mathrm{D} Y(V)[U]),
    \end{equation}
    where $Y$ is any smooth extension of $\grad\ell$ and we can pick $Y(V)=\grad\ell(V)$. Here $\mathrm{D}Y(V)$ is a linear map defined as
    \[
    \mathrm{D}Y(V)[U] = \lim_{t \to 0}\frac{Y(V+U)-Y(Y)}{t}.
    \] Since $\mc M$ is a product manifold, we can compute $\Proj_{T_{V}\mc M}$ block-wise, that is,
\begin{equation}\label{eq:decom_hess_ell}
    \Hess~\ell(V)[U] = \Proj_{T_{V}\mc M}(\mathrm{D} Y(V)[U])=[ \Proj_{T_{v_{1}} \mc M_{1} }((\mathrm{D} Y(V)[U])_{1}),\dots, \Proj_{T_{v_{N}} \mc M_{N} }((\mathrm{D} Y(V)[U])_{N})],
\end{equation}
where $(\mathrm{D} Y(V)[U])_{i}$ is the $i$-th block of $\mathrm{D} Y(V)[U]$. Noting that $Y(V)=\grad~\ell(V)=[\grad_{1}~\ell(V),\dots,\grad_{N}~\ell(V)]$, it holds that
\begin{align}
 \nonumber  ( \mathrm{D} Y(V)[U])_{i} & = \mathrm{D}\grad_{i}~\ell(V)[U]\\
    \nonumber & = \mathrm{D} (g_{i}-\frac{1}{\alpha_{i}}g_{i}^{\top}v_{i}v_{i})(V)[U] \\
   & = \mathrm{D}g_{i}(V)[U] - \frac{1}{\alpha_{i}}\mathrm{D}(g_{i}^{\top}v_{i}v_{i})(V)[U], \label{eq:decom_DY}
\end{align}
where $g_{i}$ is defined in~\eqref{eq:e-grad} and here we take it as a function of $V$.  By definition one can see that 
    \begin{equation}\label{eq:decom_DY1}
        \mathrm{D}v_{i}(V)[u_{i}]=u_{i}.
    \end{equation}  
Noting that $g_{i}(V)=\nabla_{i}\ell(V)$ is the Euclidean gradient of $\ell$ in the $i$-th block, we have
\[
[\mathrm{D}g_{1}(V)[U],\dots,\mathrm{D}g_{N}(V)[U]]=\mathrm{D}[g_{1},\dots,g_{N}](V)[U]=\mathrm{D}\nabla\ell(V)[U] = \nabla^{2}\ell(V)[U],
\]
which implies that $\mathrm{D}g_{i}(V)[U]$ is the $i$-th block of $\nabla^{2}\ell(V)[U]$, i.e., 
\begin{equation}\label{eq:decom_DY2}
    \mathrm{D}g_{i}(V)[U] = (\nabla^{2}\ell(V)[U])_{i}.
\end{equation}
Also we have
\begin{align}
 \nonumber \mathrm{D} (g_{i}^{\top}v_{i}v_{i})(V)[U] & =   \mathrm{D} (g_{i}^{\top}v_{i})(V)[U] \cdot v_{i} +  g_{i}^{\top}v_{i} \cdot \mathrm{D} (v_{i}) (V)[U]\\
\nonumber  & = \left[ \mathrm{D} (g_{i}^{\top}) (V)[U]  v_{i}+ g_{i}^{\top}\mathrm{D} (v_{i}) (V)[U] \right] \cdot v_{i}+  g_{i}^{\top}v_{i} \cdot \mathrm{D} (v_{i}) (V)[U]\\
  & =\left[ (\nabla^{2}\ell(V)[U])_{i}^{\top}v_{i} + g_{i}^{\top}u_{i}  \right]v_{i}+ g_{i}^{\top}v_{i} u_{i}.\label{eq:decom_DY3}
\end{align}
Substituting \eqref{eq:decom_DY1}, \eqref{eq:decom_DY2} and \eqref{eq:decom_DY3} into \eqref{eq:decom_DY} gives us 
\begin{align}
  \nonumber   ( \mathrm{D} Y(V)[U])_{i} & =\mathrm{D}g_{i}(V)[U] - \frac{1}{\alpha_{i}}\mathrm{D}(g_{i}^{\top}v_{i}v_{i})(V)[U]\\
    & = (\nabla^{2}\ell(V)[U])_{i} - \frac{1}{\alpha_{i}}\left[ (\nabla^{2}\ell(V)[U])_{i}^{\top}v_{i} + g_{i}^{\top}u_{i}  \right]v_{i}  - \frac{1}{\alpha_{i}}g_{i}^{\top}v_{i} u_{i}.\label{eq:decom_DY4}
\end{align}
Since $\frac{1}{\alpha_{i}}\left[ (\nabla^{2}\ell(V)[U])_{i}^{\top}v_{i} + g_{i}^{\top}u_{i}  \right]v_{i}$ is a vector along $v_{i}$, its projection into $T_{v_{i}}\mc M_{i}$ is $0$. Since $\frac{1}{\alpha_{i}}g_{i}^{\top}v_{i} u_{i}$ is a vector along $u_{i} \in T_{v_{i}}\mc M_{i}$, its projection into $T_{v_{i}}\mc M_{i}$ is itself. Hence from~\eqref{eq:decom_DY4} we have
\begin{equation}
    \Proj_{T_{v_{i}} \mc M_{i} }( ( \mathrm{D} Y(V)[U])_{i}) = \Proj_{T_{v_{i}} \mc M_{i} }( (\nabla^{2}\ell(V)[U])_{i}) -\frac{1}{\alpha_{i}}g_{i}^{\top}v_{i} u_{i},
\end{equation}
which together with \eqref{eq:decom_hess_ell} implies that
\begin{align*}
     \Hess~\ell(V)[U] & =[ \Proj_{T_{v_{1}} \mc M_{1} }((\mathrm{D} Y(V)[U])_{1}),\dots, \Proj_{T_{v_{N}} \mc M_{N} }((\mathrm{D} Y(V)[U])_{N})]\\
     & = [ \Proj_{T_{v_{1}} \mc M_{1} }( (\nabla^{2}\ell(V)[U])_{1}) -\frac{1}{\alpha_{1}}g_{1}^{\top}v_{1} u_{1},\dots, \Proj_{T_{v_{N}} \mc M_{N} }( (\nabla^{2}\ell(V)[U])_{N}) -\frac{1}{\alpha_{N}}g_{N}^{\top}v_{N} u_{N}] \\
     & = [ \Proj_{T_{v_{1}} \mc M_{1} }( (\nabla^{2}\ell(V)[U])_{1}) ,\dots, \Proj_{T_{v_{N}} \mc M_{N} }( (\nabla^{2}\ell(V)[U])_{N})] + U\Lambda \\
     & = \Proj_{T_{V}\mc M}(\nabla^{2}\ell(V)[U] )+ U\Lambda,
\end{align*}
where $\Lambda \in \R^{N \times N}$ is a diagonal matrix with entries $\Lambda_{ii} =- \frac{1}{\alpha_{i}}g_{i}^{\top}v_{i}$.
\end{proof}

Now we are ready to prove Proposition~\ref{prop:L-smooth}.
\begin{proof}[Proof of Proposition~\ref{prop:L-smooth}]
For any $U \in T_{V}\mc M $ with $\|U\|_{F}=1$, Proposition~\ref{prop:formula_Hess} gives the Riemannian Hessian of $\ell(V)$ as 
\begin{equation}\label{eq:hess_decom}
        \Hess~\ell(V)[U] = \Proj_{T_{V}\mc M}(\nabla^{2} \ell(V)[U]) + U\Lambda,
    \end{equation}
where $\Lambda \in \R^{N \times N}$ is a diagonal matrix with entries $\Lambda_{ii} =- \frac{1}{\alpha_{i}}g_{i}^{\top}v_{i}$. Since $\Proj_{T_{V}\mc M}$ is a projection and together with Lemma~\ref{lem:bound_hessian} we have
\begin{equation}\label{eq:proj_V_ub}
    \| \Proj_{T_{V}\mc M}(\nabla^{2} \ell(V)[U])  \|_{F} \leq \| \nabla^{2} \ell(V)[U]\|_{F}\leq 2+4 \left(\|H \|_{F} + \sqrt{\bar{\alpha}}\right)^{2}.
\end{equation}
For the term $U\Lambda$ we have
\begin{align*}
    \|U\Lambda \|_{F}^{2} & =  \sum_{i=1}^{N} \|\frac{1}{\alpha_{i}}g_{i}^{\top}v_{i} u_{i}\|_{2}^{2}  \\
    & = \sum_{i=1}^{N}  \frac{1}{\alpha_{i}^{2}}(g_{i}^{\top}v_{i})^{2} \|u_{i}\|_{2}^{2}& \\
    & \leq \sum_{i=1}^{N}   \frac{1}{\alpha_{i}^{2}}\|g_{i}\|_{2}^{2}\|v_{i}\|_{2}^{2} \|u_{i}\|_{2}^{2}  \\
    & = \sum_{i=1}^{N}   \frac{1}{\alpha_{i}}\|g_{i}\|_{2}^{2} \|u_{i}\|_{2}^{2} & (\|v_{i}\|_{2}^{2}=\alpha_{i})\\
    & \leq \frac{1}{\alpha_{\min}} \|\nabla \ell(V) \|_{F}^{2} \sum_{i=1}^{N} \|u_{i}\|_{2}^{2} & (\|g_{i}\|_{2}^{2}\leq \|\nabla \ell(V) \|_{F}^{2}) \\
    & = \frac{1}{\alpha_{\min}} \|\nabla \ell(V) \|_{F}^{2}, &  (\sum_{i=1}^{N} \|u_{i}\|_{2}^{2} = \|U\|_{F}^{2}=1)
\end{align*} 
where $\alpha_{\min}=\min_{i}\alpha_{i}>0$. Now, we bound $\| \nabla \ell(V)\|_{F}$ as
\begin{align}
  \| \nabla \ell(V)\|_{F} = \|-2(H+V)M^{-1} \|_{F} \leq 2 \|H+V \|_{F} \|M^{-1}\|_{2}
     \leq 2( \| H\|_{F}+\sqrt{\bar{\alpha}} ), \label{eq:nablda_l_ub}
\end{align}
which gives that
\begin{equation}\label{eq:Ulambda_ub}
    \|U\Lambda \|_{F} \leq \frac{1}{\sqrt{\alpha_{\min}}} \| \nabla \ell(V)\|_{F} \leq \frac{2}{\sqrt{\alpha_{\min}}}( \| H\|_{F}+\sqrt{\bar{\alpha}} ).
\end{equation}
Equations~\eqref{eq:hess_decom}, \eqref{eq:proj_V_ub} and \eqref{eq:Ulambda_ub} together implies that for any $U \in T_{V} \mc M$ with $\| U\|_{F}=1$, we have
\begin{equation}\label{eq:Hess_ub}
    \|\Hess~\ell(V)[U] \|_{F} \leq \|\Proj_{T_{V}\mc M}(\nabla^{2}\ell(V)[U]) \|_{F} +\|U\Lambda \|_{F}  \leq 2+4 \left(\|H \|_{F} + \sqrt{\bar{\alpha}}\right)^{2} + \frac{2}{\sqrt{\alpha_{\min}}}( \| H\|_{F}+\sqrt{\bar{\alpha}} ).
\end{equation}
According to \citet[Corollary 10.47]{ref:boumal2023introduction}, Equation~\eqref{eq:Hess_ub} implies that $\ell(V)$ is $L$-smooth with the postulated constant. This completes the proof.
\end{proof}

\begin{proof}[Proof of Proposition~\ref{prop:block_L-smooth}]
The proof is similar to the proof of Proposition~\ref{prop:L-smooth}, and here we focus on the Riemannian Hessian of $\ell$ about the component $v_{i}$ instead of $V$. The Riemannian Hessian of $\ell(V)$ with respect to $v_{i}$ is computes as
\[
\Hess_{i}~\ell(V)[u_{i}] = \Proj_{T_{v_{i}} \mc M_{i} }(\nabla^{2}_{i}~\ell(V)[u_{i}]) - \frac{1}{\alpha_{i}}g_{i}^{\top}v_{i}u_{i},
\]
which implies that
\begin{align*}
    \|\Hess_{i}~\ell(V)[u_{i}] \|_{2} & \leq \|\Proj_{T_{v_{i}} \mc M_{i} }(\nabla^{2}_{i}~\ell(V)[u_{i}]) \|_{2} + \| \frac{1}{\alpha_{i}}g_{i}^{\top}v_{i}u_{i}\|_{2} 
     \leq \|\nabla^{2}_{i}~\ell(V)[u_{i}]\|_{2} + \| \frac{1}{\alpha_{i}}g_{i}^{\top}v_{i}u_{i}\|_{2}.
\end{align*}
Following the similar arguments in the proof of Lemma~\ref{lem:bound_hessian}, we have 
\begin{equation}
   \sup_{\|u_{i} \|_{2}=1} \|\nabla^{2}_{i}~\ell(V)[u_{i}]\|_{2}  \leq 2+4(\|H \|_{F}+\bar{\alpha})^2.
\end{equation}
Following the similar arguments in the proof of Proposition~\ref{prop:L-smooth}, for any $u_{i}\in T_{v_{i}}\mc M_{i}$ with $\|u_{i}\|_{2}=1$ we have
\begin{align*}
    \|\frac{1}{\alpha_{i}}g_{i}^{\top}v_{i}u_{i}\|_{2} & \leq \frac{1}{\alpha_{i}}\|g_{i} \|_{2} \| v_{i}\|_{2} \|u_{i} \|_{2} &\\
    & \leq  \frac{1}{\sqrt{\alpha_{i}}}\|g_{i} \|_{2} &\quad (\text{since }\|v_{i}\|=\sqrt{\alpha_{i}},~\|u_{i}\|_{2}=1) \\
    & \leq \frac{1}{\sqrt{\alpha_{i}}} \|\nabla\ell(V) \|_{F} &\\
    & \leq \frac{2}{\sqrt{\alpha_{i}}} (\|H \|_{F}+\sqrt{\bar{\alpha}}),&
\end{align*}
where the last inequality comes from~\eqref{eq:nablda_l_ub}. Thus, for any $u_{i}\in T_{v_{i}}\mc M_{i}$ with $\|u_{i}\|_2=1$ it holds that
\begin{align*}
    \|\Hess_{i}~\ell(V)[u_{i}]\|_{2} & \leq \|\nabla^{2}_{i}~\ell(V)[u_{i}]\|_{2} + \| \frac{1}{\alpha_{i}}g_{i}^{\top}v_{i}u_{i}\|_{2} \\
    & \leq 2+4(\|H \|_{F}+\sqrt{\bar{\alpha}})^2 + \frac{2}{\sqrt{\alpha_{i}}} (\|H \|_{F}+\sqrt{\bar{\alpha}})=L_{i}.
\end{align*}
Combine the above bound with~\citet[Proposition 10.47]{ref:boumal2023introduction} completes the proof.
\end{proof}

We are now ready to prove Theorem~\ref{thm:convergence}.
\begin{proof}[Proof of Theorem~\ref{thm:convergence}]
\citet[Example~3]{ref:gutman2023coordinate} shows that the selection rule of block in Algorithm~\ref{alg:RBCD with Exponential Maps} satisfies the $(1,\infty)$-norm condition. Proposition~\ref{prop:L-smooth} and Proposition~\ref{prop:block_L-smooth} show that our objective function $\ell$ is both $L$-smooth and block-wise smooth. Equation~\eqref{eq:block_L-smooth2} further demonstrates that the iterates of Algorithm~\ref{alg:algorithm1} satisfies~\citet[Assumption 1]{ref:gutman2023coordinate}. Then the convergence of Algorithm~\ref{alg:RBCD with Exponential Maps} is a direct consequence of~\citet[Theorem~3]{ref:gutman2023coordinate}.
\end{proof}

\section{Experimental Details}\label{app:Experimental Details}
\subsection{Experiment Settings}

Here, we rigorously introduce where the activation steering is applied in transformer-based LMs. Transformer architectures~\citep{ref:vaswani2017attention} have become the foundation of modern LMs, achieving remarkable performance across diverse natural language processing tasks. We consider an LM with $M$ heads per layer, exhibiting the following per-block information flow for the layer $l$:
\begin{equation*}
x^{(l+1)} = \mathrm{FFN}(\mathrm{MHA}(x^{(l)})) 
= \mathrm{FFN}\Big( 
\bigoplus_{j=1}^{M} W_j^{o} \big( \mathrm{Attn}_j(x^{(l)}) \big) 
\Big),
\end{equation*} 
where $x^{(l)}$ is the input of the specific layer $l$, with $W_{j}^{o}$ denoting the output projection matrix and $\mathrm{Attn}_j$ denoting the single-head attention transformation for each head $j=1,\dots,M$. Here, $\mathrm{FFN}$, $\mathrm{MHA}$ denote the feed-forward and multi-head attention, respectively. The direct sum operator $\bigoplus$ concatenates the outputs from all attention heads before applying the feed-forward transformation. Contemporary steering methodologies operate by introducing additive perturbations to the residual stream activations, as demonstrated in prior works~\cite{turner2024steering, ackerman2024representation}.
\begin{equation}
\tilde{x}^{(l+1)} = x^{(l+1)} + v^{(l+1)},
\label{eq:activation_steering}
\end{equation}
\noindent where $v^{(l+1)}$ is the steering vector. The computation of these steering vectors follows the procedure outlined in Algorithm~\ref{alg:algorithm1}.

\subsection{Detail Evaluation for Language-Model-as-a-Judge}

We provide below the prompts to evaluate the Diversity Score and the Unique Solution Count using OpenAI GPT4.1-mini.

\textbf{Diversity Score prompt}
\vspace{0.2cm}
\lstset{
  basicstyle=\ttfamily\small,
  breaklines=true,
  frame=single
}

\begin{lstlisting}
You are given a math reasoning problem and a list of different responses (solutions) generated by a model.
Your task is to assign a float score from 0 to 1.0 that reflects the **diversity of reasoning and approaches** among the responses.
Consider differences in:
- Mathematical strategies
- Solution steps
- Structural approach
- Logical of reasoning

### Problem:
{problem}

### Responses:
{responses}

### Instruction:
Output **only** a float number between 0 and 1.0 (inclusive), rounded to two decimal places.  
Do **not** include any explanation, symbols, or text - only the score.

Example output: `0.75'
\end{lstlisting}

\vspace{0.5cm}
\noindent\textbf{Unique Solution Count prompt}
\vspace{0.2cm}
\lstset{
  basicstyle=\ttfamily\small,
  breaklines=true,
  frame=single
}

\begin{lstlisting}
You are given a math reasoning problem and a list of responses generated by a model.

Your task is to count how many **unique** responses there are.

Two responses are considered different if they use different:
- Mathematical strategies
- Solution steps
- Logical approach
- Structure of reasoning

### Problem:
{problem}

### Responses:
{responses}

### Instruction:
Output **only** the number of unique responses as an integer.  
Do **not** include any explanation, text, or symbols - just the number.

Example output: `3'
\end{lstlisting}
\vspace{0.5cm}
\noindent\textbf{Temperature setting for evaluation.} All evaluations using the language-model-as-a-judge were performed with \textbf{GPT-4.1 mini} at \texttt{temperature = 0.0}, ensuring deterministic and consistent scoring across all prompts.

\subsection{Detailed Results with Variance: Unique Solution Count and Diversity Score }
Table~\ref{tab:results-uniquesolution-variance} reports the mean and standard deviation of Unique Solution Count and Diversity Score across benchmarks and sampling temperatures. The results demonstrate that our framework can generate more unique solutions and yield superior diversity scores across most of the experimental conditions. 
\begin{table}[!]
\begin{tabular}{c|c|c|ccc|ccc|ccc}
\hline
\multirow{3}{*}{\rotatebox{90}{Metric}}                                            & \multirow{3}{*}{Model}                                                                       & \multirow{3}{*}{T} & \multicolumn{3}{c|}{AIME24}                                                                                                                                                                                & \multicolumn{3}{c|}{MATH500}                                                                                                                                                                                       & \multicolumn{3}{c}{OlympiadBench}                                                                                                                                                                                  \\ \cline{4-12} 
&                                                                                              &                        & \multicolumn{2}{c}{SPREAD}                                                                                                                  & \multirow{2}{*}{Sampling}                                    & \multicolumn{2}{c}{SPREAD}                                                                                                                  & \multirow{2}{*}{Sampling}                                            & \multicolumn{2}{c}{SPREAD}                                                                                                                  & \multirow{2}{*}{Sampling}                                            \\ \cline{4-5} \cline{7-8} \cline{10-11}
&                                                                                              &                        & $C = 1$                                                              & $C = 10$                                                             &                                                              & $C = 1$                                                              & $C = 10$                                                             &                                                                      & $C = 1$                                                              & $C = 10$                                                             &                                                                      \\ \hline
\multirow{10}{*}{\rotatebox{90}{Unique Solution Count $\uparrow$}} & \multirow{5}{*}{\begin{tabular}[c]{@{}c@{}}\rotatebox{90}{Qwen2.5-1.5B}\end{tabular}}                     & 1.0                    & \textbf{\begin{tabular}[c]{@{}c@{}}6.97\\ ($\pm$ 0.91)\end{tabular}} & \begin{tabular}[c]{@{}c@{}}6.60\\ ($\pm$ 1.02)\end{tabular}          & \begin{tabular}[c]{@{}c@{}}6.67\\ ($\pm$ 0.99)\end{tabular}  & \begin{tabular}[c]{@{}c@{}}3.14\\ ($\pm$ 0.74)\end{tabular}          & \textbf{\begin{tabular}[c]{@{}c@{}}3.21\\ ($\pm$ 0.68)\end{tabular}} & \begin{tabular}[c]{@{}c@{}}3.14\\ ($\pm$ 0.69)\end{tabular}          & \begin{tabular}[c]{@{}c@{}}6.12\\ ($\pm$ 1.24)\end{tabular}          & \begin{tabular}[c]{@{}c@{}}6.14\\ ($\pm$ 1.23)\end{tabular}          & \begin{tabular}[c]{@{}c@{}}6.14\\ ($\pm$ 1.61)\end{tabular}          \\
&                                                                                              & 0.8                    & \textbf{\begin{tabular}[c]{@{}c@{}}6.83\\ ($\pm$ 1.00)\end{tabular}} & \begin{tabular}[c]{@{}c@{}}6.43\\ ($\pm$ 1.20)\end{tabular}          & \begin{tabular}[c]{@{}c@{}}3.60\\ ($\pm$ 0.39)\end{tabular}  & \begin{tabular}[c]{@{}c@{}}3.03\\ ($\pm$ 0.72)\end{tabular}          & \textbf{\begin{tabular}[c]{@{}c@{}}3.05\\ ($\pm$ 0.70)\end{tabular}} & \begin{tabular}[c]{@{}c@{}}2.97\\ ($\pm$ 0.75)\end{tabular}          & \begin{tabular}[c]{@{}c@{}}5.99\\ ($\pm$ 1.22)\end{tabular}          & \begin{tabular}[c]{@{}c@{}}5.99\\ ($\pm$ 1.36)\end{tabular}          & \textbf{\begin{tabular}[c]{@{}c@{}}6.04\\ ($\pm$ 1.53)\end{tabular}} \\
&                                                                                              & 0.6                    & \textbf{\begin{tabular}[c]{@{}c@{}}6.40\\ ($\pm$ 1.11)\end{tabular}} & \begin{tabular}[c]{@{}c@{}}6.07\\ ($\pm$ 1.15)\end{tabular}          & \begin{tabular}[c]{@{}c@{}}3.37 \\ ($\pm$ 0.37)\end{tabular} & \textbf{\begin{tabular}[c]{@{}c@{}}2.91\\ ($\pm$ 0.74)\end{tabular}} & \begin{tabular}[c]{@{}c@{}}2.86\\ ($\pm$ 0.78)\end{tabular}          & \begin{tabular}[c]{@{}c@{}}2.90\\ ($\pm$ 0.77)\end{tabular}          & \begin{tabular}[c]{@{}c@{}}5.57\\ ($\pm$ 1.33)\end{tabular}          & \begin{tabular}[c]{@{}c@{}}5.65\\ ($\pm$ 1.26)\end{tabular}          & \textbf{\begin{tabular}[c]{@{}c@{}}5.58\\ ($\pm$ 1.67)\end{tabular}} \\
&                                                                                              & 0.4                    & \begin{tabular}[c]{@{}c@{}}6.00\\ ($\pm$ 1.10)\end{tabular}          & \textbf{\begin{tabular}[c]{@{}c@{}}6.20\\ ($\pm$ 1.05)\end{tabular}} & \begin{tabular}[c]{@{}c@{}}6.17\\ ($\pm$ 0.90)\end{tabular}  & \begin{tabular}[c]{@{}c@{}}2.73\\ ($\pm$ 0.81)\end{tabular}          & \begin{tabular}[c]{@{}c@{}}2.74\\ ($\pm$ 0.80)\end{tabular}          & \textbf{\begin{tabular}[c]{@{}c@{}}2.77\\ ($\pm$ 0.79)\end{tabular}} & \textbf{\begin{tabular}[c]{@{}c@{}}5.30\\ ($\pm$ 1.36)\end{tabular}} & \begin{tabular}[c]{@{}c@{}}5.30\\ ($\pm$ 1.40)\end{tabular}          & \begin{tabular}[c]{@{}c@{}}5.26\\ ($\pm$ 1.99)\end{tabular}          \\
&                                                                                              & 0.2                    & \begin{tabular}[c]{@{}c@{}}5.17\\ ($\pm$ 1.19)\end{tabular}          & \textbf{\begin{tabular}[c]{@{}c@{}}5.60\\ ($\pm$ 1.11)\end{tabular}} & \begin{tabular}[c]{@{}c@{}}2.97\\ ($\pm$ 0.65)\end{tabular}  & \begin{tabular}[c]{@{}c@{}}2.47\\ ($\pm$ 0.78)\end{tabular}          & \begin{tabular}[c]{@{}c@{}}2.56\\ ($\pm$ 0.82)\end{tabular}          & \textbf{\begin{tabular}[c]{@{}c@{}}2.59\\ ($\pm$ 0.83)\end{tabular}} & \begin{tabular}[c]{@{}c@{}}4.65\\ ($\pm$ 1.48)\end{tabular}          & \begin{tabular}[c]{@{}c@{}}4.69\\ ($\pm$ 1.47)\end{tabular}          & \begin{tabular}[c]{@{}c@{}}4.69\\ ($\pm$ 2.21)\end{tabular}          \\ \cline{2-12} 
& \multirow{5}{*}{\begin{tabular}[c]{@{}c@{}}\rotatebox{90}{\shortstack{Qwen2.5-Math\\1.5B- Instruct}}\end{tabular}} & 1.0                    & \begin{tabular}[c]{@{}c@{}}6.63\\ ($\pm$ 0.86)\end{tabular}          & \textbf{\begin{tabular}[c]{@{}c@{}}7.03\\ ($\pm$ 1.07)\end{tabular}} & \begin{tabular}[c]{@{}c@{}}3.67\\ ($\pm$ 0.60)\end{tabular}  & \begin{tabular}[c]{@{}c@{}}1.92\\ ($\pm$ 1.07)\end{tabular}          & \textbf{\begin{tabular}[c]{@{}c@{}}1.94\\ ($\pm$ 1.10)\end{tabular}} & \begin{tabular}[c]{@{}c@{}}1.93\\ ($\pm$ 1.06)\end{tabular}          & \textbf{\begin{tabular}[c]{@{}c@{}}4.59\\ ($\pm$ 2.34)\end{tabular}} & \begin{tabular}[c]{@{}c@{}}4.53\\ ($\pm$ 2.35)\end{tabular}          & \begin{tabular}[c]{@{}c@{}}4.57\\ ($\pm$ 2.34)\end{tabular}          \\
&                                                                                              & 0.8                    & \begin{tabular}[c]{@{}c@{}}6.63\\ ($\pm$ 1.04)\end{tabular}          & \textbf{\begin{tabular}[c]{@{}c@{}}6.73\\ ($\pm$ 1.50)\end{tabular}} & \begin{tabular}[c]{@{}c@{}}3.47\\ ($\pm$ 0.67)\end{tabular}  & \begin{tabular}[c]{@{}c@{}}1.87\\ ($\pm$ 1.05)\end{tabular}          & \textbf{\begin{tabular}[c]{@{}c@{}}1.90\\ ($\pm$ 1.07)\end{tabular}} & \begin{tabular}[c]{@{}c@{}}1.89\\ ($\pm$ 1.10)\end{tabular}          & \begin{tabular}[c]{@{}c@{}}4.28\\ ($\pm$ 2.31)\end{tabular}          & \begin{tabular}[c]{@{}c@{}}4.36\\ ($\pm$ 2.32)\end{tabular}          & \textbf{\begin{tabular}[c]{@{}c@{}}4.39\\ ($\pm$ 2.31)\end{tabular}} \\
&                                                                                              & 0.6                    & \begin{tabular}[c]{@{}c@{}}6.27\\ ($\pm$ 1.28)\end{tabular}          & \textbf{\begin{tabular}[c]{@{}c@{}}6.47\\ ($\pm$ 1.37)\end{tabular}} & \begin{tabular}[c]{@{}c@{}}3.50\\ ($\pm$ 0.62)\end{tabular}  & \begin{tabular}[c]{@{}c@{}}1.81\\ ($\pm$ 1.01)\end{tabular}          & \textbf{\begin{tabular}[c]{@{}c@{}}1.84\\ ($\pm$ 1.02)\end{tabular}} & \begin{tabular}[c]{@{}c@{}}1.82\\ ($\pm$ 1.00)\end{tabular}          & \begin{tabular}[c]{@{}c@{}}4.10\\ ($\pm$ 2.24)\end{tabular}          & \textbf{\begin{tabular}[c]{@{}c@{}}4.16\\ ($\pm$ 2.26)\end{tabular}} & \begin{tabular}[c]{@{}c@{}}4.11\\ ($\pm$ 2.24)\end{tabular}          \\
&                                                                                              & 0.4                    & \textbf{\begin{tabular}[c]{@{}c@{}}6.23\\ ($\pm$ 1.24)\end{tabular}} & \begin{tabular}[c]{@{}c@{}}5.87\\ ($\pm$ 1.13)\end{tabular}          & \begin{tabular}[c]{@{}c@{}}6.17\\ ($\pm$ 1.57)\end{tabular}  & \begin{tabular}[c]{@{}c@{}}1.77\\ ($\pm$ 0.95)\end{tabular}          & \begin{tabular}[c]{@{}c@{}}1.73\\ ($\pm$ 0.93)\end{tabular}          & \textbf{\begin{tabular}[c]{@{}c@{}}1.78\\ ($\pm$ 0.98)\end{tabular}} & \textbf{\begin{tabular}[c]{@{}c@{}}3.99\\ ($\pm$ 2.19)\end{tabular}} & \begin{tabular}[c]{@{}c@{}}3.92\\ ($\pm$ 2.19)\end{tabular}          & \begin{tabular}[c]{@{}c@{}}3.89\\ ($\pm$ 2.14)\end{tabular}          \\
&                                                                                              & 0.2                    & \textbf{\begin{tabular}[c]{@{}c@{}}5.87\\ ($\pm$ 1.45)\end{tabular}} & \textbf{\begin{tabular}[c]{@{}c@{}}5.87\\ ($\pm$ 1.28)\end{tabular}} & \begin{tabular}[c]{@{}c@{}}5.53\\ ($\pm$ 1.78)\end{tabular}  & \begin{tabular}[c]{@{}c@{}}1.70\\ ($\pm$ 0.89)\end{tabular}          & \begin{tabular}[c]{@{}c@{}}1.74\\ ($\pm$ 0.92)\end{tabular}          & \begin{tabular}[c]{@{}c@{}}1.74\\ ($\pm$ 0.93)\end{tabular}          & \begin{tabular}[c]{@{}c@{}}3.67\\ ($\pm$ 2.04)\end{tabular}          & \begin{tabular}[c]{@{}c@{}}3.65\\ ($\pm$ 2.04)\end{tabular}          & \textbf{\begin{tabular}[c]{@{}c@{}}3.68\\ ($\pm$ 2.06)\end{tabular}} \\ \hline
\multirow{10}{*}{\rotatebox{90}{Diversity Score $\uparrow$}}       & \multirow{5}{*}{\begin{tabular}[c]{@{}c@{}}\rotatebox{90}{Qwen2.5-1.5B}\end{tabular}}                     & 1.0                    & \textbf{\begin{tabular}[c]{@{}c@{}}0.37\\ ($\pm$ 0.08)\end{tabular}} & \begin{tabular}[c]{@{}c@{}}0.33\\ ($\pm$ 0.06)\end{tabular}          & \begin{tabular}[c]{@{}c@{}}0.40\\ ($\pm$ 0.09)\end{tabular}  & \textbf{\begin{tabular}[c]{@{}c@{}}0.40\\ ($\pm$ 0.05)\end{tabular}} & \begin{tabular}[c]{@{}c@{}}0.39\\ ($\pm$ 0.05)\end{tabular}          & \begin{tabular}[c]{@{}c@{}}0.39\\ ($\pm$ 0.05)\end{tabular}          & \begin{tabular}[c]{@{}c@{}}0.38\\ ($\pm$ 0.06)\end{tabular}          & \textbf{\begin{tabular}[c]{@{}c@{}}0.40\\ ($\pm$ 0.06)\end{tabular}} & \begin{tabular}[c]{@{}c@{}}0.39\\ ($\pm$ 0.06)\end{tabular}          \\
&                                                                                              & 0.8                    & \textbf{\begin{tabular}[c]{@{}c@{}}0.51\\ ($\pm$ 0.09)\end{tabular}} & \begin{tabular}[c]{@{}c@{}}0.45\\ ($\pm$ 0.08)\end{tabular}           & \begin{tabular}[c]{@{}c@{}}0.37\\ ($\pm$ 0.05)\end{tabular}  & \begin{tabular}[c]{@{}c@{}}0.42\\ ($\pm$ 0.05)\end{tabular}          & \textbf{\begin{tabular}[c]{@{}c@{}}0.44\\ ($\pm$ 0.06)\end{tabular}} & \begin{tabular}[c]{@{}c@{}}0.41\\ ($\pm$ 0.05)\end{tabular}          & \begin{tabular}[c]{@{}c@{}}0.46\\ ($\pm$ 0.06)\end{tabular}          & \begin{tabular}[c]{@{}c@{}}0.46\\ ($\pm$ 0.06)\end{tabular}          & \begin{tabular}[c]{@{}c@{}}0.46\\ ($\pm$ 0.06)\end{tabular}          \\
&                                                                                              & 0.6                    & \textbf{\begin{tabular}[c]{@{}c@{}}0.52\\ ($\pm$ 0.08)\end{tabular}} & \begin{tabular}[c]{@{}c@{}}0.46\\ ($\pm$ 0.06)\end{tabular}           & \begin{tabular}[c]{@{}c@{}}0.39\\ ($\pm$ 0.07)\end{tabular}  & \textbf{\begin{tabular}[c]{@{}c@{}}0.41\\ ($\pm$ 0.06)\end{tabular}} & \begin{tabular}[c]{@{}c@{}}0.40\\ ($\pm$ 0.06)\end{tabular}          & \begin{tabular}[c]{@{}c@{}}0.40\\ ($\pm$ 0.06)\end{tabular}          & \begin{tabular}[c]{@{}c@{}}0.46\\ ($\pm$ 0.06)\end{tabular}          & \begin{tabular}[c]{@{}c@{}}0.46\\ ($\pm$ 0.06)\end{tabular}          & \begin{tabular}[c]{@{}c@{}}0.46\\ ($\pm$ 0.06)\end{tabular}          \\
&                                                                                              & 0.4                    & \begin{tabular}[c]{@{}c@{}}0.41\\ ($\pm$ 0.05)\end{tabular}          & \textbf{\begin{tabular}[c]{@{}c@{}}0.47\\ ($\pm$ 0.08)\end{tabular}} & \begin{tabular}[c]{@{}c@{}}0.42\\ ($\pm$ 0.05)\end{tabular}  & \begin{tabular}[c]{@{}c@{}}0.36\\ ($\pm$ 0.05)\end{tabular}          & \begin{tabular}[c]{@{}c@{}}0.36\\ ($\pm$ 0.06)\end{tabular}          & \begin{tabular}[c]{@{}c@{}}0.36\\ ($\pm$ 0.06)\end{tabular}          & \begin{tabular}[c]{@{}c@{}}0.42\\ ($\pm$ 0.05)\end{tabular}          & \textbf{\begin{tabular}[c]{@{}c@{}}0.43\\ ($\pm$ 0.05)\end{tabular}} & \begin{tabular}[c]{@{}c@{}}0.42\\ ($\pm$ 0.05)\end{tabular}          \\
&                                                                                              & 0.2                    & \begin{tabular}[c]{@{}c@{}}0.31\\ ($\pm$ 0.02)\end{tabular}          & \textbf{\begin{tabular}[c]{@{}c@{}}0.34\\ ($\pm$ 0.02)\end{tabular}} & \begin{tabular}[c]{@{}c@{}}0.33\\ ($\pm$ 0.05)\end{tabular}  & \textbf{\begin{tabular}[c]{@{}c@{}}0.29\\ ($\pm$ 0.05)\end{tabular}} & \textbf{\begin{tabular}[c]{@{}c@{}}0.29\\ ($\pm$ 0.05)\end{tabular}} & \begin{tabular}[c]{@{}c@{}}0.28\\ ($\pm$ 0.05)\end{tabular}          & \begin{tabular}[c]{@{}c@{}}0.33\\ ($\pm$ 0.05)\end{tabular}          & \textbf{\begin{tabular}[c]{@{}c@{}}0.35\\ ($\pm$ 0.05)\end{tabular}} & \textbf{\begin{tabular}[c]{@{}c@{}}0.35\\ ($\pm$ 0.05)\end{tabular}} \\ \cline{2-12} 
& \multirow{5}{*}{\begin{tabular}[c]{@{}c@{}}\rotatebox{90}{\shortstack{Qwen2.5-Math\\1.5B- Instruct}}\end{tabular}} & 1.0                    & \begin{tabular}[c]{@{}c@{}}0.64\\ ($\pm$ 0.04)\end{tabular}          & \textbf{\begin{tabular}[c]{@{}c@{}}0.65\\ ($\pm$ 0.05)\end{tabular}} & \begin{tabular}[c]{@{}c@{}}0.63\\ ($\pm$ 0.05)\end{tabular}  & \begin{tabular}[c]{@{}c@{}}0.25\\ ($\pm$ 0.06)\end{tabular}          & \begin{tabular}[c]{@{}c@{}}0.25\\ ($\pm$ 0.06)\end{tabular}          & \begin{tabular}[c]{@{}c@{}}0.25\\ ($\pm$ 0.06)\end{tabular}          & \textbf{\begin{tabular}[c]{@{}c@{}}0.39\\ ($\pm$ 0.08)\end{tabular}} & \textbf{\begin{tabular}[c]{@{}c@{}}0.39\\ ($\pm$ 0.08)\end{tabular}} & \begin{tabular}[c]{@{}c@{}}0.38\\ ($\pm$ 0.08)\end{tabular}          \\
&                                                                                              & 0.8                    & \begin{tabular}[c]{@{}c@{}}0.60\\ ($\pm$ 0.06)\end{tabular}          & \textbf{\begin{tabular}[c]{@{}c@{}}0.62\\ ($\pm$ 0.06)\end{tabular}} & \begin{tabular}[c]{@{}c@{}}0.57\\ ($\pm$ 0.05)\end{tabular}  & \begin{tabular}[c]{@{}c@{}}0.23\\ ($\pm$ 0.05)\end{tabular}          & \begin{tabular}[c]{@{}c@{}}0.23\\ ($\pm$ 0.05)\end{tabular}          & \begin{tabular}[c]{@{}c@{}}0.23\\ ($\pm$ 0.06)\end{tabular}          & \begin{tabular}[c]{@{}c@{}}0.36\\ ($\pm$ 0.08)\end{tabular}          & \textbf{\begin{tabular}[c]{@{}c@{}}0.38\\ ($\pm$ 0.08)\end{tabular}} & \begin{tabular}[c]{@{}c@{}}0.37\\ ($\pm$ 0.08)\end{tabular}          \\
&                                                                                              & 0.6                    & \textbf{\begin{tabular}[c]{@{}c@{}}0.57\\ ($\pm$ 0.05)\end{tabular}} & \begin{tabular}[c]{@{}c@{}}0.55\\ ($\pm$ 0.06)\end{tabular}          & \begin{tabular}[c]{@{}c@{}}0.55\\ ($\pm$ 0.05)\end{tabular}  & \begin{tabular}[c]{@{}c@{}}0.21\\ ($\pm$ 0.05)\end{tabular}          & \textbf{\begin{tabular}[c]{@{}c@{}}0.22\\ ($\pm$ 0.05)\end{tabular}} & \begin{tabular}[c]{@{}c@{}}0.20\\ ($\pm$ 0.05)\end{tabular}          & \begin{tabular}[c]{@{}c@{}}0.34\\ ($\pm$ 0.07)\end{tabular}          & \begin{tabular}[c]{@{}c@{}}0.34\\ ($\pm$ 0.08)\end{tabular}          & \begin{tabular}[c]{@{}c@{}}0.34\\ ($\pm$ 0.07)\end{tabular}          \\
&                                                                                              & 0.4                    & \begin{tabular}[c]{@{}c@{}}0.51\\ ($\pm$ 0.05)\end{tabular}          & \begin{tabular}[c]{@{}c@{}}0.51\\ ($\pm$ 0.07)\end{tabular}          & \begin{tabular}[c]{@{}c@{}}0.51\\ ($\pm$ 0.06)\end{tabular}  & \textbf{\begin{tabular}[c]{@{}c@{}}0.19\\ ($\pm$ 0.04)\end{tabular}} & \begin{tabular}[c]{@{}c@{}}0.18\\ ($\pm$ 0.04)\end{tabular}          & \begin{tabular}[c]{@{}c@{}}0.18\\ ($\pm$ 0.04)\end{tabular}          & \textbf{\begin{tabular}[c]{@{}c@{}}0.31\\ ($\pm$ 0.07)\end{tabular}} & \begin{tabular}[c]{@{}c@{}}0.30\\ ($\pm$ 0.07)\end{tabular}          & \begin{tabular}[c]{@{}c@{}}0.31\\ ($\pm$ 0.07)\end{tabular}          \\
&                                                                                              & 0.2                    & \begin{tabular}[c]{@{}c@{}}0.46\\ ($\pm$ 0.06)\end{tabular}          & \textbf{\begin{tabular}[c]{@{}c@{}}0.52\\ ($\pm$ 0.08)\end{tabular}} & \begin{tabular}[c]{@{}c@{}}0.41 \\ ($\pm$ 0.07)\end{tabular} & \begin{tabular}[c]{@{}c@{}}0.15\\ ($\pm$ 0.04)\end{tabular}          & \textbf{\begin{tabular}[c]{@{}c@{}}0.16\\ ($\pm$ 0.04)\end{tabular}} & \begin{tabular}[c]{@{}c@{}}0.15\\ ($\pm$ 0.04)\end{tabular}          & \begin{tabular}[c]{@{}c@{}}0.26\\ ($\pm$ 0.05)\end{tabular}          & \begin{tabular}[c]{@{}c@{}}0.26\\ ($\pm$ 0.06)\end{tabular}          & \begin{tabular}[c]{@{}c@{}}0.26\\ ($\pm$ 0.05)\end{tabular}         

\end{tabular}
\caption{Mean and variance value of Unique Solution Count $\uparrow$ and Diversity Score $\uparrow$ comparison across model variants and sampling temperature (T) on three mathematical reasoning benchmarks, including AIME24 \cite{ref:aimo2024aime}, MATH500 \cite{ref:lightman2023let} and OlympiadBench \cite{ref:he2024olympiadbench}. Bold indicates the best method in each dataset.}
\label{tab:results-uniquesolution-variance}
\end{table}
\newpage
\subsection{Pass@$N$ Across Steering Layers}
We set the random seed to 42 for all experiments. We conduct experiments to study the effects of changing the steering layers on the performance of SPREAD. For simplicity, we use the AIME24 dataset because it has only 30 samples. We screen the steering layers from 1 to 28 using Qwen2.5-Math-1.5B-Instruct. Figure~\ref{fig:result-layer} gives the results measured by Pass@8.
\begin{figure*}[!ht]
    \centering
    \includegraphics[width=1.0\textwidth]{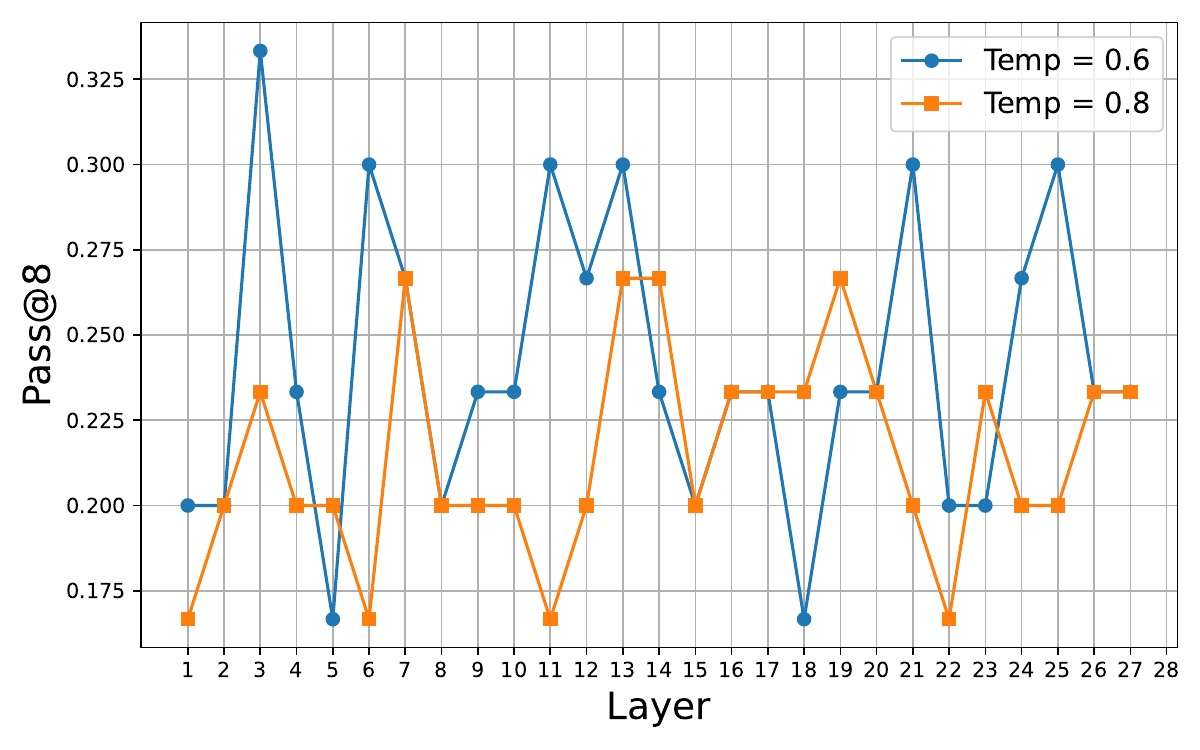}
    \caption{Layer-wise comparison of Pass@$N$ $\uparrow$ (in $\%$) for Qwen2.5-Math-1.5B-Instruct across 28 layers and different temperatures on the AIME24 benchmark.}
    \label{fig:result-layer}
    \vspace{-3mm}
\end{figure*}

\newpage
\subsection{Hypothesis Testing}
 In Section~\ref{sec:SPREAD}, we postulate that a larger volume of the parallelepiped spanned by activations $h_{im}$ will induce diversity in the final answers.  We now propose a statistical test to validate this hypothesis.
 
  For each question $q_{i},i=1,\dots,M$, we simultaneously generate $N=16$ output sequences.
  At each decoding step $\tau$, we extract the final hidden state vectors of those $N$ paths and form the matrix $H^{\tau}_{i}=[h^{\tau}_{i1},\dots,h^{\tau}_{iN}]$ and use Algorithm~\ref{alg:algorithm1} to compute steer vectors $V^{\tau}_{i}=[v^{\tau}_{i1},\dots,v^{\tau}_{iN}]$. For each question $q_i$, we we draw $10$ subsamples without replacement of size $\tilde{N}=8 $ from $N$ responses, yielding $10$ sub-sampled paths per question. Let $k\in\{1,\dots,10\}$ index these.

For each sub-sampled path $k$ of question $q_{i}$, define the observed diversity
    \[
    U_{ik} = \#~\text{the unique solutions among $\tilde{N}$ responses},
    \]
and the cumulative “volume” proxy
    \[
    Z_{ik} = \sum_{\tau=1}^{\tau_{final}}\log\det[I+(H^{\tau}_{ik}+V^{\tau}_{ik})^{\top}(H^{\tau}_{ik}+V^{\tau}_{ik})],
    \]
    where $ H^{\tau}_{ik}$ and $V^{\tau}_{ik}$ are activations and steering vectors corresponding to sub-sampled path $k$ of question $q_{i}$.
    
Since $U_{ik}$ is a count between $1$ and $\tilde{N}$, we model $U_{ik}$ as a Binomial outcome with $\tilde{N}$ trials and success probability $p_{ik}$ so that
    \[
    U_{ik} \sim \mathrm{Binomial}(\tilde{N},p_{ik}),
    \]
    where $p_{ik}$ links with $Z_{ik}$ with a logistic  regression:
    \[
    p_{ik}=\frac{\exp(\beta_{0,i}+\beta Z_{ik})}{1+\exp(\beta_{0,i}+\beta Z_{ik})} = \frac{1}{1+\exp(-\beta_{0,i}-\beta Z_{ik})},
    \]
    where $\beta_{0,i} \in \mathbb{R}$ are question-specific intercepts and $\beta \in \mathbb{R}$ is the common slope parameter.
We are interested in testing whether the predictor $Z$ is positively associated with the outcome $U$:
\[
\mathcal H_0: \beta \le 0
\qquad \text{versus} \qquad
\mathcal H_A: \beta > 0.
\]
We estimate $\{\beta_{0,i}\}_{i=1}^{M}$ and $\beta$ by maximizing the binomial log-likelihood:
\[
\max_{\{\beta_{0,i}\}_{i=1}^{M},\beta } \sum_{i=1}^M \sum_{k=1}^{10}   U_{ik}(\beta_{0,i} + \beta Z_{ik}) - \tilde{N} \log\!\big(1 + e^{\beta_{0,i} + \beta Z_{ik}}\big) .
\]

Inference is based on a Wald statistic with a cluster-robust (sandwich) variance estimator, clustering at the question level to account for within-question dependence and heteroskedasticity:
\[
z_{\mathrm{Wald}} = \frac{\hat\beta}{\widehat{\mathrm{SE}}(\hat\beta)}.
\]
Applying this procedure yields $\hat\beta=0.88$, clustered $\widehat{\mathrm{SE}}(\hat\beta)=0.24$, so $z_{\mathrm{Wald}}=3.6$ and a one-sided $p=0.001$. Thus, we reject $H_{0}$ at conventional levels. Substantively, a one-unit increase in $Z$ multiplies the odds of a response being unique by $\exp(\widehat \beta)\approx2.4$, indicating a strong positive association between the proposed volume measure and answer diversity.
\end{document}